\documentclass[12pt]{article}

\usepackage{arxiv}
\usepackage{booktabs}
\usepackage{tabularx}
\usepackage{array}

\usepackage[most]{tcolorbox} \newtcblisting{promptbox}[1][]{ colback=gray!5, colframe=gray!50, title=#1, breakable, listing only, listing options={ breaklines=true, breakatwhitespace=true, columns=fullflexible, basicstyle=\small\ttfamily } }

\title{Reliability and Effectiveness of Autonomous AI Agents in Supply Chain Management}
\date{}
\author{
Carol Xuan Long\\
Harvard University/ ETH AI Center \\
\And
David Simchi-Levi\\
MIT/Purdue\\
\And
Feng Zhu\\
MIT\\
\AND
Huangyuan Su\\
Harvard University/Kempner Institute\\
\And 
Andre P. Calmon \\
Georgia Tech
\And
Flavio P. Calmon\\
Harvard University\\
}

\begin{document}

\maketitle


\begin{abstract} 
This paper studies the performance and reliability of autonomous generative AI agents in multi-echelon supply chains using the MIT Beer Game. We examine how model choice, operational guardrails, centralized data sharing, and prompt design affect system performance. In our best-performing configuration, GenAI agents reduce total supply-chain costs by up to 80\% relative to human teams.
Despite strong average performance, autonomous agents can exhibit substantial run-to-run instability, generating volatile procurement decisions and large tail costs. We characterize this phenomenon as \textbf{agent bullwhip}, the amplification of decision instability in autonomous multi-agent systems. We show that this instability can propagate across echelons and compound over time, even when the underlying demand path is held fixed.
We then evaluate two approaches for improving reliability: reinforcement-learning post-training and operational guardrails. Both reduce tail events and mitigate agent bullwhip, but they operate through different mechanisms and require different levels of information and model access. Reinforcement-learning post-training delivers the largest gains in reliability and system performance when system-level feedback is available, while guardrails provide a simple training-free alternative for constraining extreme decisions.

\end{abstract}

\section[Introduction]{Introduction\protect\footnotemark}
\footnotetext{The paper expands and provides more technical details on the concepts and framework described in a recent article, Long, C., Simchi-Levi, D., Calmon, A. P., \& Calmon, F. P. When supply chains become autonomous. Harvard Business Review \citep{long2025supply}.}\label{sec:Intro}

Experts suggest that a fully autonomous supply chain, where AI makes all inventory and logistics decisions, may be close at hand. They predict that autonomous supply chains will soon deliver significant gains in productivity, efficiency, and responsiveness. Fueling this excitement is the rapid progress in Large Language Models (LLMs), the engine behind modern Generative AI (GenAI), which can now handle tasks ranging from procurement decisions and revenue management to logistics.
So far, however, most LLM applications in supply chains have focused on narrow tasks, using a single model for a specific task, such as demand forecasting or replenishment decisions \citep{menache2025generative}. The broader vision is more ambitious: a fully autonomous supply chain in which multiple LLM agents collaborate, each executing distinct responsibilities across the network.
But how close are we to this reality? To find out, we reimagined the Beer Game --- a classic supply chain management simulation developed in the 1960s and used by countless management education programs --- by replacing every human player with a GenAI agent. This setup allows us to analyze the supply-chain management capabilities of GenAI agents and the impact of lead time, information sharing, and financial constraints on agent performance. Our self-contained supply-chain simulation serves as a testbed for benchmarking GenAI agents against human experts, anticipating integration challenges, and identifying strategies for using this technology effectively in supply-chain management.

We show that making an autonomous, GenAI-managed supply chain work depends on mastering four critical levers: the model you choose, the policies and guardrails you set, the information you share, and the instructions you give. We benchmarked AI performance against that of humans: we used data from 12 Georgia Tech cohorts with more than 100 students in total who played the Beer Game over the past three years, all operating under the same system conditions as the GenAI testbed. In our best-performing setup --- using Qwen-3 4B with optimized prompts, data-sharing rules, and guardrails --- the AI agents reduced average costs across 30 replications of the game by as much as 80\% relative to the student teams. While the use of students rather than experienced supply chain managers as a benchmark may raise concerns, a substantial body of prior research demonstrates that experimental outcomes in the Beer Game are comparable between the two groups \citep{ancarani2016measuring, croson2006behavioral, machuca2004impact}.

Of course, average performance alone is insufficient to justify operational deployment in real-world applications. Supply chain practitioners must evaluate not only expected cost outcomes but also system reliability. An autonomous policy that yields low average costs but occasionally produces highly volatile procurement, production, or ordering decisions is practically unviable. This is especially true in multi-echelon supply chain networks, where localized errors can rapidly propagate, compounding single deviations into distorted demand signals for upstream agents over time. 

Standard evaluations of LLM agents often emphasize average task performance or aggregate cost, but supply-chain deployment also requires consistency across repeated executions of the same operational environment. 
The reliability issue of autonomous agents is especially salient because the underlying LLMs are stochastic and can produce inconsistent decisions across repeated runs of the same prompt. Our analysis reveals that out-of-the-box GenAI agents can be efficient on average yet unreliable: occasional ordering decisions can lead to high total supply-chain costs. We show that post-training the agents on synthetic supply-chain data dramatically improves reliability while maintaining a significant cost advantage over human decision makers.


To capture the unreliability risk in using LLM agents for decision-making, we introduce the concept of the \emph{agent bullwhip} effect: the amplification of decision instability and unreliability across runs in multi-agent systems. This instability manifests along two dimensions. First, at a fixed point in time, decision variance increases across facilities as one moves upstream: retailers tend to produce relatively stable orders, while wholesalers, distributors, and factories exhibit progressively larger dispersion and more severe tail decisions. Second, within the same facility, decision variance can grow over time as early ordering differences alter inventory positions, backlogs, and shipment pipelines, causing small behavioral deviations to compound through delayed feedback and increasing the risk of erratic ordering policies over time. We further distinguish the \emph{decision bullwhip}: the component of agent bullwhip that remains after conditioning on the demand path. This framing separates variability inherited from external demand from variability generated by the agent's own policy, and it explains why a system can be efficient on average yet unreliable across repeated executions. Because repeated sampling does not eliminate this decision-driven component, we develop two reliability interventions to address agent bullwhip: GRPO-based post-training and inference-time guardrails.

These findings point toward a future in which AI handles routine operational decisions while offering cost efficiency and flexible availability, thereby creating capacity for human experts to pursue higher-level strategic challenges in supply chains.

The main contributions of this paper are threefold:

\begin{enumerate}
\item We develop the GenAI Beer Game as a controlled testbed for autonomous multi-echelon supply-chain decision making and benchmark LLM agents against human teams. We identify four key levers—model selection, policies and guardrails, information sharing, and prompt engineering—with model capability emerging as the dominant factor. In our best-performing configuration, GenAI agents reduce costs by up to 80\% relative to human teams.

\item We introduce \emph{agent bullwhip} and \emph{decision bullwhip} to characterize and isolate reliability risks in autonomous supply chains. We show that decision instability can amplify both across facilities and over time, even under an identical demand path. 

\item We develop two approaches to improve reliability. GRPO-based post-training uses system-level supply-chain rewards to learn more stable and coordinated decision policies, substantially reducing bullwhip and tail events while improving efficiency. When system-level rewards are unavailable, inference-time guardrails provide a practical training-free alternative that can be applied directly to the base model to constrain extreme actions. Both approaches reduce tail events and mitigate agent bullwhip.

\end{enumerate}

\section{Setup and Related Work}\label{sec:Method}

\subsection{The Beer Game: A Timeless Lesson in Supply Chain Dynamics}
The Beer Distribution Game is a canonical system-dynamics environment for studying feedback, delay, and boundedly rational decision-making in supply chains \citep{forrester1961industrial,sterman1989modeling}. Its enduring lesson is that local decisions can generate system-level instability even when each participant is trying to behave rationally. The Beer Game is therefore a useful testbed for autonomous AI agents because it turns a simple decision interface into a dynamic coordination problem with delayed feedback.

In the Beer Game, four players operate a simple, serial supply chain with a retailer, a wholesaler, a distributor, and a factory. Each week, every player makes one decision: how much to order from their upstream partner. The setup is straightforward, but the constraints are revealing. Players must balance the cost of holding excess inventory against the penalty for backorders --- unfulfilled orders that must be shipped later. The structure of the beer supply chain complicates this central trade-off. Players operate in silos and cannot communicate, and only the retailer sees the actual end-customer demand. Significant built-in delays exist for both orders and shipments, and it typically takes two weeks for a shipment to arrive. This creates "pipeline inventory" --- beer that has been ordered but is not yet on hand --- which many players fail to account for. The players’ shared goal is to meet demand at the lowest possible total cost for the entire supply chain.

This structure produces the classical bullwhip effect: order signals become distorted and amplified as demand information moves upstream. A small, temporary fluctuation in customer demand creates wild swings in upstream agents' orders.  
%
A longstanding literature studies the bullwhip effect as the upstream amplification of order variability caused by information distortion in supply chains. Lee et al. \citep{lee1997bullwhip, lee1997information} introduce the phenomenon and identify key drivers, including demand signal processing, order batching, price fluctuations, and rationing behavior. Chen et al.~\citep{chen2000quantifying} quantify how forecasting rules, lead times, and information sharing affect the magnitude of bullwhip in a simple supply chain, and show that exponential smoothing forecasts can further amplify order variability depending on lead times and forecasting parameters~\citep{chen2000exponential}. Building on this literature, we study an agent-driven bullwhip from the perspective of across-run reliability: autonomous LLM agents can introduce decision variability even when demand paths and operational states are held fixed.

\subsection{GenAI in Supply Chains and the Beer Game}

Recent work on Generative AI (GenAI) in supply chain management studies language models for demand forecasting, procurement support, replenishment planning, and managerial decision support \citep{menache2025generative}. A related stream examines LLM-based or foundation-model agents for inventory management and autonomous supply chains, emphasizing natural-language interfaces, interpretable coordination, and flexible decision support across organizational boundaries \citep{simchi2025large, quan2024invagent,xu2024multi,xu2024implementing, zheng2025llms,qi2026economy}. Much of this literature, however, focuses on specialized applications or architectures that assume substantial system design around the model.

Our setting is closer to how firms are likely to deploy frontier models: through standard interfaces, with limited ability to retrain closed-weight models. We therefore ask whether off-the-shelf GenAI agents can manage a dynamic multi-echelon supply chain when each agent controls one role in the Beer Game. 

The setup of this paper is based on a recent implementation of an AI-powered version of the Beer Game \citep{long2025genaibeergame, long2025supply}. The GenAI Beer Game is similar to the classic version, but replaces human players with LLM agents (e.g., GPT-5). Each agent takes on a single role, such as the wholesaler, and makes ordering decisions autonomously. Like human players, AI agents manage inventory, respond to downstream orders, and submit upstream orders. In contrast to many AI benchmarks that test a single LLM’s performance on a task, the GenAI Beer Game examines the agents’ ability to coordinate as a group.

In Section \ref{sec:inference-time_methods_efficiency}, we focus on "inference-time methods" to improve the effectiveness of off-the-shelf LLMs. We consider approaches that optimize how these models are used rather than changing the models themselves. Inference-time methods include crafting better instruction prompts for the agents, orchestrating information flow between agents, and designing simple rules or policies that limit what actions agents can take. Unlike existing work \citep{boussioux2025socratic} that studies prompting-based reasoning interventions for improving LLM decisions in the Beer Game, we focus on fully autonomous agent deployment, evaluating both inference-time methods and a reinforcement-learning post-training framework.

\subsection{Scaling Test-Time Compute}
We consider scaling test-time compute in the form of repeated sampling to enhance the reliability of the LLM agents and reduce tail risks. Our inference-time analysis connects to work on test-time compute, where LLM performance is improved by sampling multiple candidate responses, applying self-consistency or voting, or allocating more inference compute to harder instances \citep{wang2022self,wu2024inference,snell2024scaling,brown2024large}. These methods are attractive because they are computationally efficient, requiring no post-training of LLMs. A Beer Game decision differs from a static response to a self-contained prompt: it is a context-dependent decision whose consequences propagate through inventories, backlogs, and shipment pipelines. This distinction motivates our empirical test of repeated sampling in Section~\ref{sec:intervention_mechanisms}. If instability primarily reflects decoding noise, aggregation should stabilize decisions; if it reflects a weakness in the underlying policy, improving reliability requires a stronger intervention.


\subsection{Reinforcement Learning Post-training for Reliable Agents}
The reliability problem also connects to reinforcement-learning post-training for LLMs. Policy-gradient methods such as Proximal Policy Optimization provide a general framework for improving stochastic policies through reward feedback \citep{schulman2017proximal}. Group Relative Policy Optimization (GRPO) replaces a learned value critic with relative comparisons across sampled outputs, making it attractive when reward evaluation is easier than value estimation \citep{shao2024deepseekmath}. Recent reasoning models show that reinforcement learning can induce more reliable behavior when rewards are aligned with the task objective \citep{guo2025deepseekr1}. In supply chains, rewards are operationally observable through holding cost, backlog cost, and total system cost. Therefore, the Beer Game provides a useful post-training environment in which local actions have delayed consequences and system-level rewards measure coordination across echelons.

\subsection{Decision Reliability and Tail Risk in Autonomous Supply Chains}

There is growing interest in evaluating decision-making policies beyond expected performance, including their distributional behavior, tail outcomes, and realized sample paths \citep{simchi2025simple,simchi2023regret,zhu2026adaptive}. This perspective is especially important in supply chains, where decisions are sequential, coupled across facilities, and exposed to demand uncertainty and lead-time delays. A policy with low expected cost may still generate unacceptable realized trajectories, such as inventory depletion, backlog accumulation, excessive order volatility, or persistent service failures.

This concern has roots in robust and risk-aware inventory theory. Classical distribution-free inventory models protect against poor realizations when the demand distribution is only partially known \citep{scarf1958minmax,gallego1993distribution}, while robust optimization approaches extend this logic to dynamic inventory and supply-chain systems \citep{bertsimas2006robust}. A parallel literature in sequential decision-making studies tail-sensitive and constraint-aware policies, including CVaR optimization \citep{rockafellar2000optimization,chow2015risk}, safe reinforcement learning \citep{garcia2015comprehensive}, and constrained policy optimization \citep{achiam2017constrained}. Our setting differs from this work in that reliability risk is generated not only by exogenous demand uncertainty or model misspecification, but also by stochastic LLM decisions that interact through multi-echelon feedback.

This paper surfaces the unreliability risks of autonomous AI agents arising from stochasticity that can produce inconsistent decisions across repeated runs of the same task. Standard evaluations of LLM agents often emphasize average task performance or aggregate cost. However, deployment of LLM agents in supply chains requires consistency across repeated executions of the same operational environment. Because LLM decisions are stochastic, the same state can induce different actions across runs; in a multi-echelon network, those differences can propagate upstream and compound over time. We refer to this amplification of run-to-run decision instability as the agent bullwhip effect. It links decision reliability to the structure of the supply chain: reliability risk is not only a property of an individual model response, but also an emergent property of delayed, decentralized coordination. This framing motivates both our empirical reliability analysis and our use of post-training to learn more stable supply-chain policies.



\section{Inference-time Methods: Lessons Learned Using GenAI as Autonomous Agents without Training} \label{sec:inference-time_methods_efficiency}
In this section, we evaluate various types of LLMs and examine inference-time methods that can improve the performance of LLM agents in supply-chain decision-making. Here, inference-time strategies reflect how most people interact with GenAI agents: by submitting natural-language queries. Because LLMs are trained on heterogeneous data sources, they can generate meaningful responses with little or no task-specific customization. This raises two central questions: Can these models be deployed ``as is'' to manage complex supply-chain tasks effectively? If not, are there generalizable inference-time strategies that firms can use to steer GenAI models toward better performance?

We first note the generational gap between earlier and more recent LLMs. Recent models are increasingly equipped with reasoning capabilities that enable them to decompose complex decision problems into smaller, more tractable steps and use explicit intermediate reasoning to guide action selection. This capability is particularly relevant in supply-chain settings, where local ordering decisions interact dynamically with inventory positions, backlogs, delays, and upstream amplification. We report the performance of various LLMs in the Beer Game and benchmark them against human teams. Our results show that more advanced models with stronger reasoning capabilities substantially outperform earlier model generations. At the same time, earlier models can perform effectively when supported by appropriate inference-time interventions, particularly curated information sharing and coordination through a central orchestrator.

Our experiments identify four inference-time levers that are critical to the autonomous use of GenAI in supply chains: (1) model selection, (2) policies and guardrails, (3) data sharing through a centralized orchestrator, and (4) prompt engineering. The most important lever is the selection of a capable model. However, even the most advanced models require careful deployment and guidance. Performance can be substantially improved by sharing the right data through a centralized orchestrator to enhance coordination and by using appropriately designed prompts. Together, these four levers determine whether the integration of autonomous agents succeeds or fails.

Using all four levers, we demonstrate that GenAI agents outperform human teams operating under identical conditions, providing evidence for their potential integration into real-world supply-chain operations.

\subsection{Generational Leap of LLMs}
The choice of LLM used for the agents is the single most important determinant of performance because an agent’s underlying reasoning capability directly affects supply chain costs and system stability. Less advanced models can amplify system noise into costly bullwhip effects, whereas more advanced models can attenuate such amplification. To evaluate reliability, we conducted multiple identical Beer Game runs for each model. In our decentralized setup --- that is, a setting in which no information is shared across agents --- we found that many earlier-generation models were highly inefficient, producing pronounced bullwhip effects and generating costs an order of magnitude higher than those of human teams. These models were also unreliable: across identical runs, total costs varied substantially, ranging from 13\% to 46\% of the mean.

More concerningly, some models failed to follow instructions, leading to systemic breakdowns. In our trials, models such as Microsoft’s Phi-4 and DeepSeek-R1-0528 violated basic ordering rules in more than 25\% of cases.

However, recent models with advanced reasoning capabilities demonstrate a clear improvement in performance. For example, upgrading agents from GPT-4o mini to GPT-5 mini reduced total supply chain costs by 70\%. Similarly, the newer and more lightweight Llama 4 Maverick 17B model substantially outperformed its much larger predecessor, Llama 3.3 70B, reducing costs by 82\%, although its results remained unstable. Among the reasoning models evaluated, Qwen-3 4B delivers the strongest performance, achieving more than a 65\% reduction in total supply-chain costs relative to both GPT-5 mini and Llama 4 Maverick 17B. These findings indicate that firms should prioritize reasoning ability and instruction-following when selecting a model; subsequent interventions should be viewed as performance optimizations.

\subsection{Policies and Guardrails to Limit Costly Errors}
\label{subsec:guardrail}
Policies and guardrails are most valuable not because they make agents inherently better at solving the underlying decision problem, but because they prevent agents from taking high-cost actions when their reasoning or forecasts fail. Constraints on an agent’s range of possible actions can therefore materially improve both efficiency and reliability. This is particularly important in supply-chain settings, where simple guardrails can prevent panic-induced over-ordering that triggers costly bullwhip cascades across upstream suppliers. In our experiments, a simple budget constraint proved especially effective. Each agent was assigned a fixed budget, and orders were not allowed to exceed available funds. This hard constraint operates as a brake on panic buying: when an agent experiences a stockout and attempts to place an excessively large order, the budget limit forces a more measured response, reducing demand amplification and limiting the propagation of shocks upstream.

The effects were substantial: total costs decreased by 25\% for GPT-5 mini, 39\% for GPT-4o mini, and 41\% for Llama 4 Maverick 17B. For capable but less stable models such as Llama 4 Maverick 17B, the improvement in reliability was also pronounced, with cross-run variation in performance declining from 46\% to 37\%. Although Qwen-3 4B already performs strongly out of the box, imposing a budget constraint yields an additional 3\% reduction in total supply-chain costs. These findings suggest that, once a sufficiently capable model has been selected, firms should adopt targeted operational policies to limit erratic behavior. A hard budget constraint represents a particularly high-leverage intervention, as it directly restricts the excessive ordering behavior that contributes to instability.

\subsection{Information Orchestration: Share Curated Data Through a Central Orchestrator}
To evaluate how information sharing affects agent performance, we introduced a central “orchestrator”: an agent with full visibility across the supply chain and responsibility for sharing specific, curated information with the agents participating in the game. This design reflects an important distinction between human and LLM-based decision-making. Information that is useful to human teams may distract an AI agent, resulting in poorer decisions and higher costs. Accordingly, we evaluated two information-sharing strategies in which the orchestrator shared information but did not make decisions. The results indicate that more data is not necessarily better.

In the first scenario, the orchestrator shared only real-time customer demand. When agents received the current week’s customer demand, performance improved across all models. Total costs decreased by approximately 18\% for GPT-5 mini, 25\% for Llama 4 Maverick 17B, 7\% for Qwen-3 4B, and 38\% for GPT-4o mini.

In the second scenario, agents received both demand history and analysis. Specifically, we provided a five-week demand history and volatility analysis. The results were mixed. This richer information significantly improved performance for less advanced models, with costs for GPT-4o mini decreasing by 69\%. However, for more advanced models, the additional information appeared to act as a distraction, and performance was worse than when agents received only real-time demand.

Notably, other data elements that typically benefit human decision makers --- such as inventory position or pipeline inventory --- provided limited benefits and often exacerbated the bullwhip effect. These findings suggest that firms should be selective and empirically test which data are shared with AI agents. For more advanced models, less information is often more effective.

\subsection{Prompt Design}
Because LLMs are probabilistic systems, task framing matters. Prompt design can substantially improve the performance of less advanced models, although it may provide limited benefits for more advanced models. In our experiments, reframing the objective --- the instruction provided to the LLM --- from the general goal of “minimize total costs” to the more specific objective of “minimize the weighted average of backlog and holding costs” generated large gains for less advanced models. This prompt revision reduced costs by 44\% for GPT-4o mini and 33\% for GPT-4.1 mini. For more advanced models, the effect of rephrasing the task objective was negligible.

These results indicate that prompt design should be used as a secondary performance lever. It can unlock meaningful improvements in less advanced models, but it is not a substitute for strong reasoning capabilities, robust guardrails, and curated data sharing.

\subsection{Summary: Reasoning Models and Critical Design Levers Achieve Above-Human Performance}
Benchmarking against human performance is a critical test for autonomous systems. Demonstrating competitiveness with trained professionals validates the potential for integrating AI agents into real-world supply chain operations. This competitive performance, together with AI agents’ inherent advantages in cost efficiency and continuous availability, provides a clear motivation for adoption.

Our most significant finding is that GenAI agents powered by state-of-the-art models can manage a supply chain at a level of proficiency that not only rivals, but can exceed, that of human experts. We benchmarked multiple GenAI agent configurations against historical performance data from 12 Georgia Tech cohorts, comprising more than 100 students operating the same Beer Game. The results were striking. As illustrated in Figure \ref{fig:inference_time_methods_bullwhip_human_comparison} (Left), earlier-generation models exhibit suboptimal performance out of the box, roughly doubling the supply-chain costs achieved by humans. However, when provided with appropriate information through an orchestrator, they can outperform human teams by reducing costs by 32\%.

More advanced models with reasoning capabilities, by contrast, achieve competitive performance even when deployed out of the box. As shown in Figure \ref{fig:inference_time_methods_bullwhip_human_comparison}, Qwen-3 4B and GPT-5 mini achieved a 77\% and 33\% cost reduction relative to human teams when operated out of the box. Most notably, when these models are optimized using additional strategic levers --- specifically enhanced information sharing and policy constraints such as budget limitations --- AI agents powered by GPT-5 mini, Llama 4 Maverick 17B, and Qwen-3 4B achieved 50\% to 80\% reductions in total supply-chain costs relative to their human counterparts, as shown in Figure \ref{fig:inference_time_methods_bullwhip_human_comparison} (Right).

\begin{figure}[t]
    \centering
    \includegraphics[width=\textwidth]{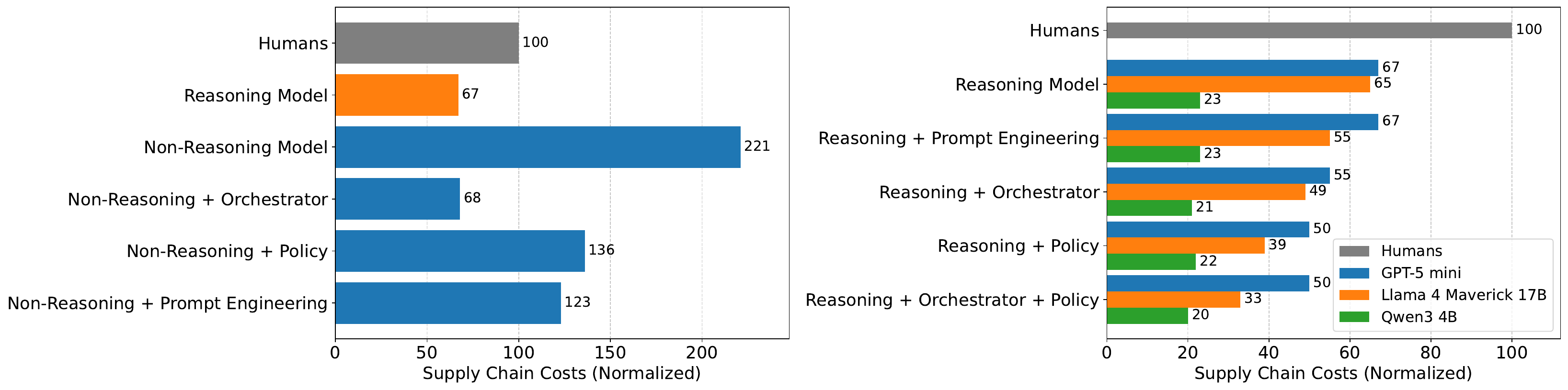}
    \caption{Comparisons of AI setups against human teams in supply-chain cost performance. Massive leapt in performance between new-generation reasoning models compared to non-reasoning models. The out-of-the-box
reasoning model (left, GPT-5 mini) exceeded human-level performance, while its non-reasoning predecessor (GPT-4o mini)
required policy constraints, orchestration, and prompt engineering to close the gap with humans. Optimized with the same techniques, reasoning models (right, GPT-5 mini, Llama 4 Maverick 17B, and Qwen-3 4B) achieved up to a 80\% cost reduction relative to human teams in the MIT Beer Game.}
    \label{fig:inference_time_methods_bullwhip_human_comparison}
\end{figure}

This finding has important implications. It demonstrates that autonomous AI agents are already capable of handling the complex, dynamic decision-making required for core supply chain functions. By delegating such operational tasks to reliable GenAI agents, human managers can redirect their attention from day-to-day operational routines toward higher-value activities, such as strategic network design, supplier relationship management, navigating major disruptions, and breaking down the functional silos that currently separate supply chain, finance, sales, and trade. In this setting, the role of the supply chain professional evolves from operator to strategist.

Overall, these findings suggest that GenAI agents, when deployed with state-of-the-art models and appropriate configuration, can achieve decision-making quality comparable to that of human experts. This represents a pivotal opportunity to accelerate the transition toward autonomous supply chains and to redeploy human expertise toward more strategic and creative supply chain challenges.

Although the results in Figure~\ref{fig:inference_time_methods_bullwhip_human_comparison} are encouraging, they report average costs over identical simulation runs and therefore obscure a reliability question that is central to real-world deployment: do agents perform consistently well across runs? To examine this issue, Table~\ref{tab:supply_chain_costs_cv} reports detailed statistics across 30 runs of the 25-week beer game for each AI setup, including mean total supply-chain cost, standard deviation, and coefficient of variation. The results reveal substantial variability. All three models exhibit high coefficients of variation, ranging from 26\% to 45\% of the mean. This level of run-to-run instability poses a significant operational risk: even when an agent performs well on average, firms may still face occasional but costly failures that disrupt inventory planning, amplify upstream orders, and undermine trust in autonomous supply-chain management. 

We investigate the reliability of autonomous supply-chain agents through a range of interventions in the next section. When system-level supply-chain cost signals are available for training, reinforcement-learning post-training combined with operational guardrails provides the strongest improvement in reliability. When such global feedback is unavailable, inference-time guardrails offer a particularly practical alternative: they require no additional model training, can be layered directly onto an existing model, and still provide an effective mechanism for constraining unstable or high-cost actions.

\begin{table}[ht]
    \centering
    \caption{Supply Chain Cost and Coefficient of Variation of 25-week Beer Game Across Runs by Model Type}
    \label{tab:supply_chain_costs_cv}
    \begin{adjustbox}{width=\textwidth}
        \begin{tabular}{llrrr}
            \toprule
            \textbf{Decision-Maker} & \textbf{Total Costs} & \textbf{Max Costs across Runs} & \textbf{Standard Deviation} & \textbf{Coefficient of Variation} \\
            \midrule
            GPT-5 mini default           & 3927 & 8644 & 1772 & 45\% \\
            Llama 4 Maverick 17B default & 4026 & 8912 & 2103 & 52\% \\
            Qwen-3 4B default & 1585 & 2847 & 406 & 26\% \\
            \bottomrule
        \end{tabular}
    \end{adjustbox}    
\end{table}

\section{Reliability Issues of Autonomous AI Agents: the Agent Bullwhip Effect}

As demonstrated in the preceding sections, autonomous GenAI agents, when optimally configured, can achieve strong average performance in supply-chain settings, potentially surpassing human teams. However, mean performance alone is insufficient to justify operational deployment in real-world applications. Supply-chain practitioners must evaluate not only expected cost outcomes but also system consistency, robustness, and exposure to tail risks. An autonomous policy that yields low expected costs but occasionally produces highly volatile procurement, production, or replenishment decisions is practically unviable.

This concern becomes particularly salient when the same autonomous supply-chain system is run repeatedly under identical conditions. Because LLMs are inherently probabilistic, even advanced models generate different decisions across identical runs. In our experiments, we ran LLM-powered GenAI agents on the same GenAI Beer Game environment across 30 repeated trials. If the agents used a stable inventory policy, one would expect these repeated runs to produce similar ordering trajectories. Instead, as shown in Table~\ref{tab:supply_chain_costs_cv}, there is significant run-to-run variation in agents' orders, even though the environment, prompts, and system structure are all held fixed. This variability should not be dismissed as a mere technical artifact of language models; rather, it constitutes a critical operational vulnerability.

The risk is especially pronounced in multi-echelon networks, where localized ordering errors can propagate rapidly and compound over time, transforming isolated deviations into distorted demand signals for upstream agents. Instability at a single node can therefore cascade through the network, generating costly tail outcomes for the supply chain as a whole.

\subsection{Agent Bullwhip Effect}

As described in Section~\ref{sec:Method}, the classical bullwhip effect refers to the amplification of order quantities as one moves upstream in the supply chain. Let $q_{k,t}^{(r)}$ denote the order placed by facility $k$ in period $t$ during simulation run $r$. The classical bullwhip effect can be expressed as an increase in order variance across echelons:
\[
    \mathcal{B}_i^{(r)}
    =
    \frac{\operatorname{Var}_{t}\!\left(q_{k,t}^{(r)}\right)}
    {\operatorname{Var}_{t}\!\left(q_{k-1,t}^{(r)}\right)}
    > 1;
\]
when downstream variance is zero, any positive upstream variance is likewise classified as bullwhip.
In the Beer Game, a modest change in customer demand can therefore turn into much larger order swings at the wholesaler, distributor, and factory.
Our findings suggest that when autonomous LLM agents make these decisions, there is a second layer of amplification. Not only do order levels amplify upstream, but the dispersion and tail risk of the decision itself also amplify across otherwise identical runs. 

\begin{definition}[Agent bullwhip]
\label{def:agent-bullwhip}
Consider a discrete-time serial supply chain with \(n\) echelons indexed by \(k=1,\dots,n\), where tier \(0\) represents external customer demand.  Consider repeated runs \(r=1,\ldots,R\) of the same supply-chain environment under an identical demand path, system configuration, and agent setup. Let \(q_{k,t}^{(r)}\) denote the order placed by echelon \(k\) in period \(t\) during run \(r\). Define the run-to-run variance of echelon \(k\)'s order in period \(t\) as
\[
    \sigma_{k,t}^{2}
    =
    \operatorname{Var}_{r}\!\left(q_{k,t}^{(r)}\right).
\]

We say that \emph{agent bullwhip} occurs when decision instability, measured across repeated runs, is amplified by the supply-chain system. This amplification can manifest along two dimensions: across echelons within a fixed time period, and over time within a fixed echelon.
First, run-to-run decision variance increases upstream at a fixed period \(t\):
\[
    \sigma_{k,t}^{2} > \sigma_{k-1,t}^{2}.
\]
Equivalently, for adjacent echelons, define
\[
    \Psi_k(t)
    =
    \frac{
    \operatorname{Var}_{r}\!\left(q_{k,t}^{(r)}\right)
    }{
    \operatorname{Var}_{r}\!\left(q_{k-1,t}^{(r)}\right)
    }.
\]
A value of \(\Psi_k(t)>1\) indicates that run-to-run decision variance is amplified as orders move from echelon \(k-1\) to echelon \(k\). Because fixed-demand experiments can have zero run-to-run variance at tier \(0\), we avoid normalizing by \(\operatorname{Var}_r(q_{0,t}^{(r)})\). Instead, when adjacent denominators are positive, cumulative adjacent-tier amplification can be summarized by
\[
    C_j(t)
    =
    \prod_{k=1}^{j} \Psi_k(t).
\]

Second, \emph{intertemporal agent bullwhip} occurs when run-to-run decision variance increases over time within the same echelon:
\[
    \sigma_{k,t+1}^{2} > \sigma_{k,t}^{2}.
\]
Equivalently, define the within-echelon amplification ratio
\[
    \Phi_k(t)
    =
    \frac{
    \operatorname{Var}_{r}\!\left(q_{k,t+1}^{(r)}\right)
    }{
    \operatorname{Var}_{r}\!\left(q_{k,t}^{(r)}\right)
    },
\]
with \(\Phi_k(t)>1\) indicating that decision instability accumulates over time for facility \(k\).
\end{definition}

\begin{figure}[hbp]
    \centering
    \includegraphics[width=0.85\textwidth]{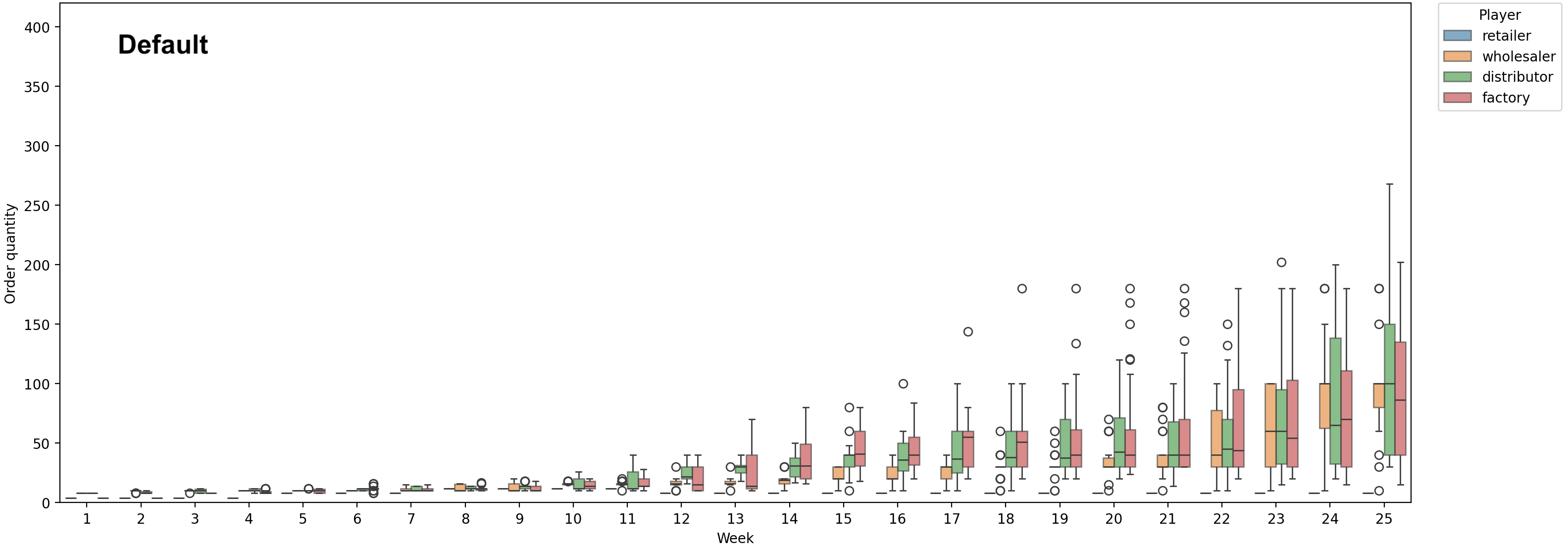}
    \caption{\textbf{Agent bullwhip: order variability across agents and time.} 
    For each week and facility, the colored box captures the middle 50\% of orders across repeated runs, the center line denotes the median, the whiskers show the non-outlier range beyond the interquartile range, and circles represent outlier orders. The amplification of decision unreliability across echelons manifests along two dimensions: decision variance increases across facilities at a fixed point in time and within each facility over time.
    }
    \label{fig:second_order_bullwhip_default}
\end{figure}

\subsection{Agent Bullwhip in Action}

To demonstrate the agent bullwhip effect in practice, we analyze order quantities across 30 runs of the GenAI Beer Game driven by Qwen-3 4B under the same demand path, game structure, and prompt. We plot the resulting order variability across facilities and weeks in Figure \ref{fig:second_order_bullwhip_default}. Because demand is identical across runs, the dispersion in the figure reflects conditional decision instability rather than demand variability. Across all weeks, the retailer exhibits minor volatility in order quantities, but this variation becomes substantially larger for the wholesaler and most pronounced for the distributor and the factory.

We can analyze the agent bullwhip along \emph{two dimensions}, as run-to-run order variance compounds across agents and time. First, holding a given week fixed, such as Week 15 in Figure~\ref{fig:second_order_bullwhip_default}, order variability is barely visible for the retailer but increases as one moves upstream to the wholesaler, distributor, and factory. 
%
Second, agent bullwhip accumulates intertemporally within a fixed facility. For a given facility, such as the distributor shown in green, order variability increases over time as early decision shocks perturb inventory positions, backlogs, and pipeline inventories. These perturbed states then feed into later ordering decisions, allowing small differences across runs to persist and compound through lead times and delayed feedback. Thus, even under identical customer demand paths and system configurations, instability compounds both across agents and over time. 

This result is counter to a common intuition about autonomous AI agents: agents do not become more reliable merely by interacting with the environment and making repeated decisions. The second dimension of agent bullwhip, order variability over time within a fixed facility, shows that decisions can become increasingly unstable across repeated interactions, even when the environment and agent configuration are held fixed and cost feedback is provided. This suggests that explicit intervention mechanisms are needed to support reliable autonomous supply-chain decision-making.

The traditional bullwhip effect captures amplification in realized order patterns. It shows that, in a given game, upstream nodes may order more aggressively than downstream nodes in response to demand shocks. However, it does not capture whether an agent's decisions are stable across repeated exposure to the same environment. Agent bullwhip captures this missing component. The variance of decisions across repeated runs reflects uncertainty in the task as perceived and processed by the agent. When that variance grows upstream and over time, it indicates that the autonomous system is not merely reacting strongly to demand; it is becoming less reliable in how it interprets and responds to the same operational state.

This result raises an important caution: a supply chain governed by autonomous LLM agents may appear effective when judged by average cost alone, yet remain operationally fragile if it exhibits high decision variability. In practice, reliability is a first-order performance criterion. Firms are unlikely to delegate planning decisions to an agent that occasionally makes unreasonable choices, even if its average performance is strong.

Firms must be able to trust that a planning system will produce similar recommendations when faced with the same inputs, particularly in settings where supply-chain decisions create downstream financial commitments and upstream production responses. The presence of agent bullwhip suggests that reliability risk is endogenous to the multi-agent supply-chain structure: uncertainty is not only present at the individual-agent level, but is also amplified by the network and by time delays. This creates a new class of tail risk for autonomous operations and complicates the case for deploying off-the-shelf LLM agents without additional safeguards or specialized training.

The central challenge, therefore, is to identify mechanisms that improve agent reliability. We begin with a training-free approach commonly used to reduce unreliability arising from model stochasticity: repeated sampling at test time.

\subsection{Repeated Sampling Fails to Improve Reliability}
\label{sec:intervention_mechanisms}

\begin{figure}[ht]
    \centering
    \begin{minipage}{1\textwidth}
        \centering
        \begin{minipage}{0.9\textwidth}
            \centering
        \includegraphics[width=\textwidth]{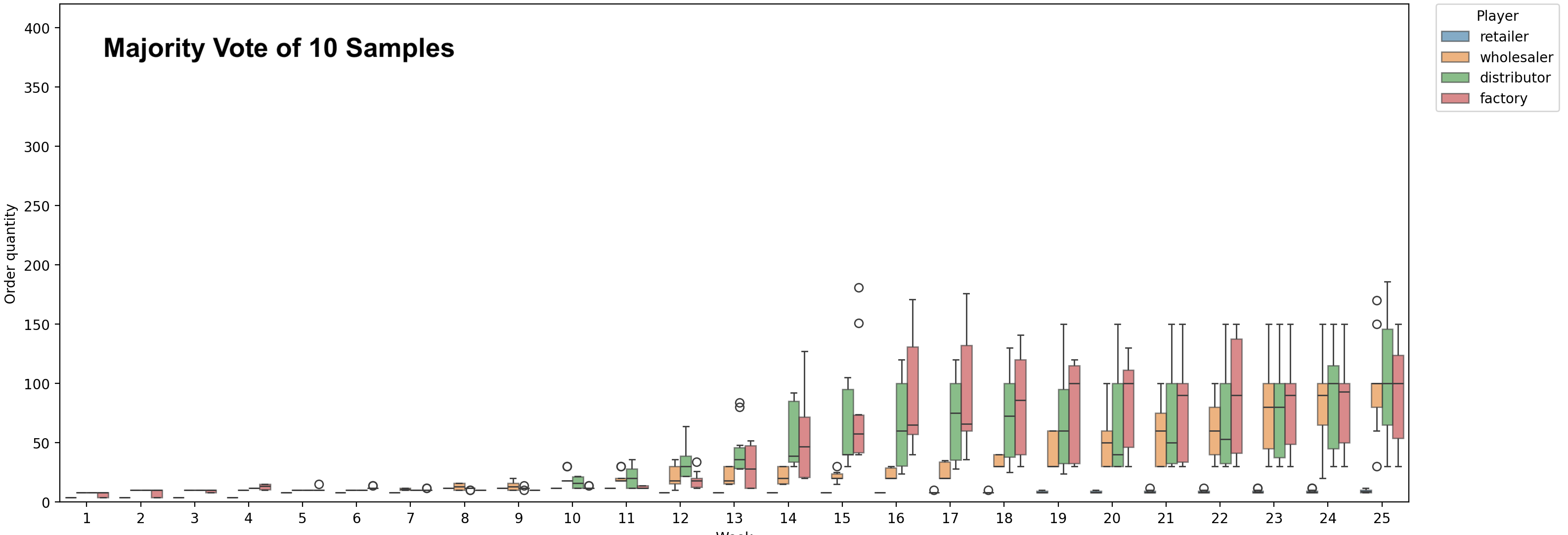}
        \end{minipage}
        \hfill
        \begin{minipage}{0.9\textwidth}
            \centering
            \includegraphics[width=\textwidth]{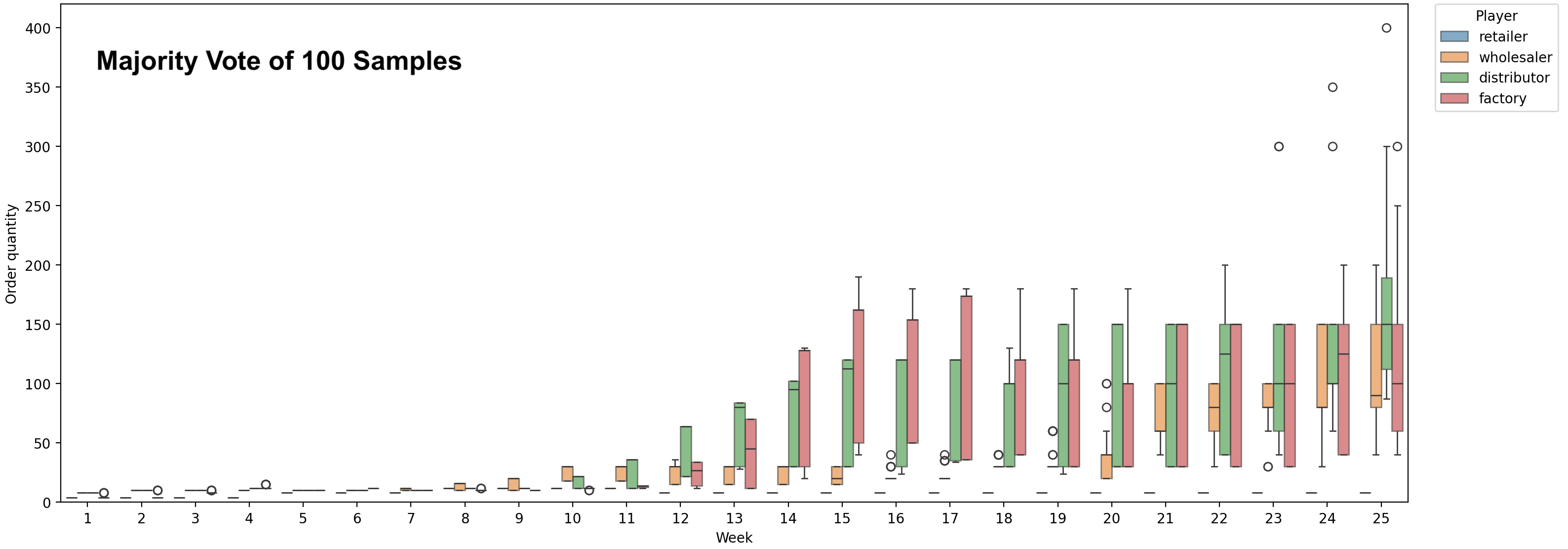}
        \end{minipage}
    \end{minipage}
    \caption{\textbf{Effect of repeated sampling on agent bullwhip.} The top panel reports results in which each order decision is determined by majority vote over 10 independent samples, while the bottom panel uses 100 samples. Increasing test-time sampling does not reduce run-to-run variability, indicating that decision instability requires policy-level intervention, such as reinforcement-learning post-training of LLM agents.}
    \label{fig:test-time-scaling}
\end{figure}


A common approach in the computer science literature to address model unreliability at test time, without model post-training, is to introduce redundancy by drawing multiple outputs and aggregating them, often via majority voting~\citep{wang2022self}. If order instability were driven primarily by random decoding noise, such ensembling should induce policy convergence and stabilize performance. To test this hypothesis, we compare the default single-sample baseline (Figure~\ref{fig:second_order_bullwhip_default}) with majority-voting schemes based on 10 and 100 samples.

Figure~\ref{fig:test-time-scaling} shows that repeated sampling fails to mitigate unreliability. Here, we report results based on the mode across multiple samples, with the median across decisions producing similar results. Substantial run-to-run variation persists, and tail events remain pronounced. This reflects a deeper issue: instability arises from suboptimal decision policies rather than incidental randomness. Off-the-shelf models lack a stable inventory-control policy and exhibit systematic errors, such as overreacting to backlogs or neglecting pipeline inventory. When the model is structurally uncertain, additional samples can simply reproduce the same deficient reasoning.

More broadly, test-time ensembling is insufficient for dynamic operational settings. While redundancy can reduce incidental noise, supply-chain management requires structured decision rules that account for delayed information and intertemporal trade-offs.

The next subsection sharpens this intuition by separating demand-driven variability from decision-driven variability. This distinction clarifies why repeated sampling may leave substantial residual instability: it can average over some local randomness, but it does not teach the agent a stable replenishment policy. Achieving reliability therefore requires modifying the underlying policy, which motivates our reinforcement-learning post-training approach.

\subsection{Understanding Agent Bullwhip}

This subsection explains how the agent bullwhip effect arises and why repeated sampling fails to mitigate it. The goal is not to claim that LLM decisions follow a particular inventory-control formula. Rather, the goal is to separate two sources of order variability that are confounded in realized trajectories: variability inherited from customer demand and variability created by the agent's own decision policy.

\paragraph{Agent bullwhip decomposition.} We distinguish two sources of randomness that contribute to the agent bullwhip effect. The first is the external demand path \(D=\{D_t\}_{t\ge0}\). The second is the agent's decision randomness \(\epsilon=\{\epsilon_{k,t}\}\), which summarizes run-specific variation in how tier \(k\) interprets the same operational state and converts it into an order. We use \(\epsilon\) as a reduced-form representation of stochastic LLM behavior; it need not correspond to an explicit perturbation chosen by the model. Consider a discrete-time serial supply chain with \(n\) echelons indexed by \(k=1,\dots,n\), where tier \(0\) represents external customer demand and \(q_{0,t}=D_t\). When conditioning on a realized demand path, we write \(D=d\).

The distinction matters because the classical bullwhip effect concerns the propagation of demand uncertainty, whereas autonomous agents can also introduce decision uncertainty even when the demand path and operational state are held fixed. In LLM decision-making, repeated sampling or best-of-\(n\) selection can reduce some output noise, but any residual policy-level variation still enters the supply-chain dynamics. The law of total variance separates these two channels. For each tier and period,
\begin{equation}
\operatorname{Var}(q_{k,t})
=
\underbrace{\operatorname{Var}_{D}\!\left(\mathbb E_{\epsilon}[q_{k,t}\mid D]\right)}_{V^D_{k,t}}
+
\underbrace{\mathbb E_{D}\!\left[\operatorname{Var}_{\epsilon}(q_{k,t}\mid D)\right]}_{V^\epsilon_{k,t}}.
\label{eq:total-variance-components}
\end{equation}

\paragraph{A two-facet perspective.} The first term, \(V^D_{k,t}\), is the demand-driven component: it measures how external demand fluctuations propagate upstream after averaging over the agent's internal randomness. The second term, \(V^\epsilon_{k,t}\), is the decision-driven component: it measures run-to-run instability generated by autonomous agents when the demand path is held fixed. We call upstream amplification of \(V^D_{k,t}\) the \emph{demand bullwhip}, and upstream amplification of \(V^\epsilon_{k,t}\) the \emph{decision bullwhip}.

This distinction is important for intervention design. Classical bullwhip mitigation mechanisms, such as demand smoothing, order coordination, and information sharing, primarily target the demand-driven component \(V^D_{k, t}\). They can reduce the variability inherited from external demand, but they do not eliminate the decision-driven component \(V^\epsilon_{k, t}\), which is generated by the agent's own policy even after conditioning on a fixed demand path. In the context of LLM agents, this means that two runs with the same demand realization and the same operational state may still produce different order-up-to targets, and these residual differences are then propagated upstream by the supply-chain feedback loop.

\paragraph{Decision bullwhip may dominate.} Decision bullwhip can dominate in regimes where agents are tuned for prediction stability. In \Cref{sec:theory-bullwhip}, we discuss a stylized inventory model following \citep{chen2000quantifying, chen2000exponential}. Our results (\Cref{thm:demand-bullwhip-lower-bounds} and \Cref{thm:decision-bullwhip-lower-bounds}) show that when forecasts change smoothly, demand-driven amplification can be relatively mild over a moderate number of tiers. By contrast, residual decision variance can still accumulate across tiers. If decision variances are of comparable magnitude across facilities, then agent-level randomness can become a major source of order variability even when demand itself is stable.

This also highlights why repeated sampling is an incomplete remedy. Majority voting may reduce the local decision variance, but unless it drives this variance very close to zero, the remaining decision noise continues to enter the feedback system and propagate upstream. The structural source of the agent bullwhip is therefore not merely stochasticity at a single decision point; it is the \emph{interaction} between decision variability, lead times, information delays, and decentralized replenishment decisions inherent in autonomous supply chains.

\paragraph{Example: LLM performance under constant demands.} We next isolate the decision channel empirically by running the same LLM-powered supply chain under constant customer demand across repeated trials. In this experiment, exogenous demand contributes no run-to-run variability, so any dispersion in orders comes from the agent policy and from how early decision differences propagate through inventories, backlogs, and pipeline shipments. Figure~\ref{fig:second_order_bullwhip_constant} shows that substantial order dispersion remains even in this fixed-demand setting. The pattern mirrors the agent bullwhip documented above: variability is modest near the retailer but becomes more pronounced upstream and over time. This provides direct evidence of decision bullwhip, because the instability cannot be attributed to changes in customer demand.

\begin{figure}[ht]
    \centering
    \includegraphics[width=0.85\textwidth]{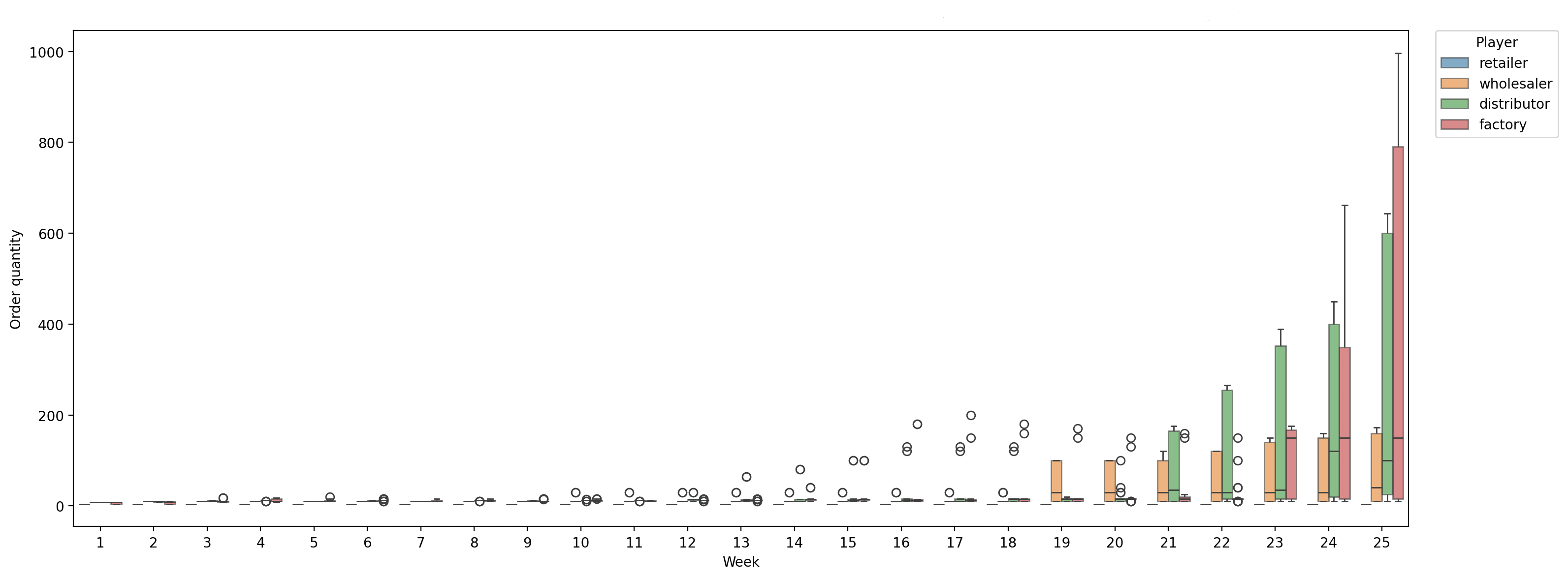}
    \caption{\textbf{Decision bullwhip: order variability across agents and time under constant demands.} 
    Customer demand is held constant as $4$ across repeated runs, so order dispersion measures conditional decision instability rather than demand variability. Persistent dispersion across facilities and weeks indicates that stochastic LLM decisions can generate a decision bullwhip even in the absence of demand shocks.
    }
    \label{fig:second_order_bullwhip_constant}
\end{figure}

\paragraph{Implication for LLM deployment.} In practical deployments of LLM agents for multi-echelon supply-chain management, the observed agent bullwhip effect should therefore be understood as the combined outcome of demand-driven and decision-driven amplification. The demand channel captures how external uncertainty propagates through the system, while the decision channel captures how run-to-run instability in agent policies is generated and amplified even under a fixed demand path. Addressing the latter requires changing the decision policy itself, rather than merely averaging over its outputs. This motivates our subsequent section, which explores various policy intervention mechanisms. 


\section{Policy Interventions for Reliable Autonomous Supply Chains}

The reliability analysis above indicates that agent bullwhip is not merely a surface-level sampling artifact. The fixed-demand experiment isolates a decision-driven component of order variability, and the failure of repeated sampling shows that averaging outputs is not enough to produce a stable replenishment policy. The problem lies deeper, in the absence of a decision policy that can reliably coordinate across echelons and optimize system-level outcomes. 

We study two complementary approaches that differ in implementation overhead: first, we propose a reinforcement learning framework that trains supply-chain-specialized agents using Group Relative Policy Optimization (GRPO), which learns more stable and coordinated inventory-management policies; second, we revisit training-free guardrails that directly constrain the agents' feasible actions.

\subsection{Training Supply-Chain-Specialized Agents with Group Relative Policy Optimization (GRPO)}
\label{sec:grpo}


This pattern is consistent with the nature of the task. Off-the-shelf, general-purpose LLMs exhibit broad linguistic competence and strong generalized reasoning, but they are not explicitly trained to internalize the dynamics of inventory replenishment, including lead times, delayed feedback loops, multi-echelon coordination, and system-wide cost trade-offs. This challenge is analogous to robotics, where models often require task-specific training before they can operate reliably in unfamiliar environments. In the supply-chain setting, the absence of such specialization leads to erratic heuristic responses, substantial run-to-run variability, and severe tail risks.

The natural way to turn a general-purpose LLM into an inventory-management agent is reinforcement-learning post-training. In this framing, the LLM is a stochastic policy whose input is the supply-chain state observed by a player and whose output is that player's ordering decision. Post-training refines the LLM's parameters so that an agent using the LLM produces effective decisions on the trajectories the policy itself induces.

This motivates the central question: can reinforcement-learning post-training transform a general-purpose LLM from a capable but volatile decision-maker into a reliable inventory-management agent? More specifically, can feedback from realized supply-chain performance induce a replenishment policy that is both cost-effective and robust across runs?

Applying standard reinforcement-learning algorithms to the Beer Game is difficult. Most modern policy-gradient methods, such as actor-critic algorithms, rely on a learned value function that maps states to expected discounted returns, and several features of this environment make such a value function unreliable: each agent's state includes inventory, backlog, pipeline, and recent orders; each agent sees only its local state and the orders and shipments from its immediate neighbors; rewards are delayed by lead times; and the relevant horizon spans many weeks.

We therefore use Group Relative Policy Optimization (GRPO), which avoids learning an explicit value function. At each training step, GRPO samples a group of trajectories under the current policy and computes a baseline from the realized costs of the group. Trajectories that outperform the baseline are reinforced; trajectories that underperform are discouraged. This relative comparison provides a stable learning signal in a high-dimensional, multi-agent, delayed-feedback environment where value estimation would be unreliable. 

Our GRPO-based framework trains a single shared LLM backbone across the four tiers using system-level rewards, so that the shared model learns how local ordering decisions interact across echelons. After post-training, we deploy independent instances of this shared backbone at each tier, with each agent acting on its own local state but using the shared learned policy. The trained LLM therefore carries a global view of how every tier should behave under system-level cost, even though each deployed instance acts only on local information. This is the centralized-training, decentralized-execution paradigm from multi-agent reinforcement learning: coordination is built into the policy through training rather than into runtime communication or orchestration.

\paragraph{Environment and global post-training setup.}
We consider the GenAI Beer Game described in Section \ref{sec:Method}, a four-echelon supply chain consisting of a retailer, wholesaler, distributor, and factory. All four agents share the same LLM policy $\pi_\theta$ but operate at different positions in the supply chain. The system evolves over $T$ weeks. Every week, every LLM-powered agent observes its local state (inventory, backlog, incoming shipments, etc.) and outputs an order quantity.

Training is performed over multiple simulated episodes. Each episode corresponds to one full beer game trajectory of length $T$. For each training step, we run $G$ independent episodes (rollouts) using the current policy.

\paragraph{Demand curriculum.}
To expose the model to diverse dynamics, we train under synthetic demand distributions:
\begin{itemize}
    \item \textbf{Curriculum 1 (Poisson):} Demand is drawn i.i.d. from $\mathrm{Poisson}(\lambda)$, where $\lambda \sim \mathcal{U}(5,20)$ is resampled per episode.
    \item \textbf{Curriculum 2 (Truncated Normal):} Demand is drawn i.i.d. from a truncated normal distribution:
    $D_t \sim \mathrm{TruncNormal}(\mu, \sigma^2; [0,50]), \quad
    \mu \sim \mathcal{U}(8,20), \ \sigma \sim \mathcal{U}(2,6)$.
\end{itemize}
In practice, training may use either distribution or a curriculum that switches between them across training steps.

\medskip\noindent\textbf{Cost structure and reward signals.}
\par\noindent
Let $c_{k,t}$ denote the cost incurred by agent $k \in \{\text{retailer}, \text{wholesaler}, \text{distributor}, \text{factory}\}$ at week $t$. This cost consists of holding and backorder components:
\[
c_{k,t} = c_{\mathrm{hold}} I_{k,t} + c_{\mathrm{back}} B_{k,t}.
\]
Here, $I_{k,t}$ and $B_{k,t}$ denote on-hand inventory and backlog, respectively. Aggregating across agents, the total system cost at week $t$ is $c_{t}^{\mathrm{sys}} = \sum_a c_t^{(a)}$, and the cumulative system cost over an episode across $T$ weeks is
\[
C^{\mathrm{sys}} = \sum_{t=1}^T c_t^{\mathrm{sys}}.
\]
Similarly, each agent incurs a cumulative cost:
\[
C_{k} = \sum_{t=1}^T c_{k,t}.
\]

The reward signal is defined along two dimensions: reward scope and reward attribution. The reward scope determines whether performance is measured at the system level or at the level of individual agents. Under a system-level objective, the reward is given by $r = -C^{\mathrm{sys}}$, so that all agents share a common signal reflecting total supply chain efficiency. Under an agent-level objective, each agent $a$ is evaluated using its own cumulative cost, $r_k = -C_k$, which emphasizes decentralized performance.

The reward attribution determines how costs are assigned to individual decisions over time. Under episode-level attribution, a single scalar reward is assigned uniformly to all actions taken within a trajectory, so that each decision receives the same signal $r$ (or $r_k$). In contrast, under rollout (return-to-go) attribution, each decision taken at week $t$ is assigned the cumulative downstream cost incurred from that point onward:
\[
r_t = -\sum_{\tau=t}^T c_\tau,
\]
where $c_\tau$ corresponds to either $c_\tau^{\mathrm{sys}}$ or $c_{k,\tau}$ depending on the chosen reward scope. This formulation attributes credit to actions based on their long-term impact on future costs, allowing earlier decisions to be evaluated according to the downstream consequences they induce.

\paragraph{Group relative policy optimization (GRPO).} 
Group Relative Policy Optimization (GRPO) is well-suited to the supply-chain setting because it avoids the need to learn a separate value function or critic over a high-dimensional, partially observed, and multi-agent state space. Instead, it uses the group of sampled trajectories as an implicit baseline. At each training step, multiple episodes are generated under the same environment, and their realized costs are compared. Trajectories that achieve lower costs than their peers are reinforced, while those with higher costs are discouraged. This relative evaluation provides a stable learning signal without requiring explicit value estimation.

\paragraph{Advantage construction.}
This comparison is formalized through a group-normalized advantage, which we denote by \(\mathrm{Adv}_{k,t}^{(i)}\), where \(i\) indexes the episode, \(t\) indexes the week, and \(k \in \mathcal{A}\) indexes the agent. Let \(r_{k,t}^{(i)}\) denote the reward assigned to agent \(k\) at week \(t\) in episode \(i\). The advantage is computed by normalizing this reward across the \(G\) episodes collected in the same training step:
\[
\mathrm{Adv}_{k,t}^{(i)}
=
\frac{
r_{k,t}^{(i)} - \frac{1}{G}\sum_{j=1}^G r_{k,t}^{(j)}
}{
\sqrt{
\frac{1}{G}\sum_{j=1}^G
\left(
r_{k,t}^{(j)}
-
\frac{1}{G}\sum_{h=1}^G r_{k,t}^{(h)}
\right)^2
}
+
\varepsilon_{\mathrm{norm}}
}.
\]
The small constant \(\varepsilon_{\mathrm{norm}}>0\) prevents division by zero and is distinct from the decision-shock notation in Section~\ref{sec:theory-bullwhip}. This normalization ensures that learning depends on relative performance within the group rather than the absolute scale of realized costs. For reward-to-go signals, normalization is performed separately for each week and agent, so that a decision is compared only with decisions made by the same agent at the same temporal position across episodes. Episode-level and system-level rewards are recovered as special cases: under episode-level attribution, \(r_{k,t}^{(i)}\) is constant across \(t\), while under system-level scope, the same system reward is shared across agents.

\paragraph{Policy representation.}
The LLM defines a stochastic policy $\pi_\theta$ that maps supply-chain contexts to ordering decisions. Each episode $i$ produces a sequence of weekly decisions $y^{(i)} = (y_{1}^{(i)}, \ldots, y_{T}^{(i)})$, where each $y_{t}^{(i)}$ is a vector of actions across the four agents,
$y_{t}^{(i)}
=
\big(y_{r,t}^{(i)},\, y_{w,t}^{(i)},\, y_{d,t}^{(i)},\, y_{f,t}^{(i)}\big)$.
Each component corresponds to the order quantity placed at week $t$ by the retailer, wholesaler, distributor, and factory, respectively. The four agents make separate decisions based on their local supply-chain contexts. Conditional on these contexts, the joint likelihood factorizes across weeks and agents:
\[
\log \pi_\theta(y_i \mid x)
=
\sum_{t=1}^{T}
\sum_{k \in \mathcal{A}}
\log \pi_\theta\big(y_{k,t}^{(i)} \mid x_{k,t}^{(i)}\big),
\]
where $\mathcal{A} = \{\text{retailer, wholesaler, distributor, factory}\}$ and \(x_{a,t}^{(i)}\) denotes the context available to agent \(a\) at week \(t\) in episode \(i\). Each term is computed from the token-level log-probabilities of the generated response for that agent.

\paragraph{Objective function.}
GRPO updates the shared policy by reinforcing agent-week decisions that perform well relative to comparable decisions in the same training group. The corresponding objective is
\begin{equation}
\begin{aligned}
\mathcal{J}_{\mathrm{GRPO}}(\theta)
=\mathbb{E}
\left[
\frac{1}{G}
\sum_{i=1}^{G}
\frac{1}{T|\mathcal{A}|}\sum_{t=1}^{T}
\sum_{k \in \mathcal{A}}
\mathrm{Adv}_{k,t}^{(i)}
\log \pi_\theta\big(y_{k,t}^{(i)} \mid x_{k,t}^{(i)}\big) -\beta
D_{\mathrm{KL}}(\pi_\theta \| \pi_{\mathrm{ref}})\right].    
\end{aligned}
\end{equation}

Thus, positive advantages increase the likelihood of the corresponding agent decisions, while negative advantages decrease it. Since all agents share the same LLM backbone, gradients from all agents and all weeks are aggregated into a single parameter update. The KL penalty stabilizes training by constraining the updated policy to remain close to the frozen reference model.

\paragraph{Model update.}
Each training step consists of two phases: data collection and policy optimization. First, $G$ beer game episodes are simulated using the current policy with stochastic sampling, producing trajectories of decisions and associated cost signals. Based on the chosen reward scope and attribution scheme, advantages are computed either at the episode level or as reward-to-go signals and normalized across the group.

Second, the model parameters are updated via stochastic gradient ascent on the GRPO objective. In practice, for each trajectory $i$, we compute the mean token-level log-probability and weight it by the corresponding advantage:
\begin{equation}
\label{eq:model_update}
\begin{aligned}
\nabla_\theta \mathcal{J}
\approx
\frac{1}{G T |\mathcal{A}|}
\sum_{i=1}^{G}
\sum_{t=1}^{T}
\sum_{k \in \mathcal{A}}\mathrm{Adv}_{k,t}^{(i)}
\nabla_\theta
\log \pi_\theta
\big(
y_{k,t}^{(i)} \mid x_{k,t}^{(i)}
\big) -
\beta \nabla_\theta
D_{\mathrm{KL}}
\big(
\pi_\theta \| \pi_{\mathrm{ref}}
\big).
\end{aligned}
\end{equation}
Gradients are accumulated across all trajectories, optionally regularized with a KL penalty against a frozen reference model, clipped to control variance, and applied using an optimizer such as AdamW. This procedure iteratively shifts the policy toward generating decisions that achieve lower supply-chain costs relative to competing episodes. 

\begin{remark}[Pooling reward information across agents]
Importantly, the model update pools reward information from all four agents, as shown in equation~\eqref{eq:model_update}. Although agents are deployed independently at different echelons, they share a common base model during training. As a result, the update is not based on the experience of a single facility in isolation; instead, it aggregates learning signals from the retailer, wholesaler, distributor, and factory. This allows the shared policy to internalize how local ordering decisions affect system-wide outcomes across the supply chain, while preserving decentralized execution at deployment.
\end{remark}

\paragraph{Evaluation protocol.}
After training, the learned policy is evaluated on the standard MIT Beer Game demand pattern:
\[
D_t = (4,4,4,4,8,\ldots,8).
\]
Performance is measured over 30 independent simulation runs, and we report the average total supply chain cost as well as per-agent costs.

\subsection{Global GRPO Post-training Results}\label{sec:Results}

\begin{figure}[ht]
    \centering
    \includegraphics[width=0.85\textwidth]{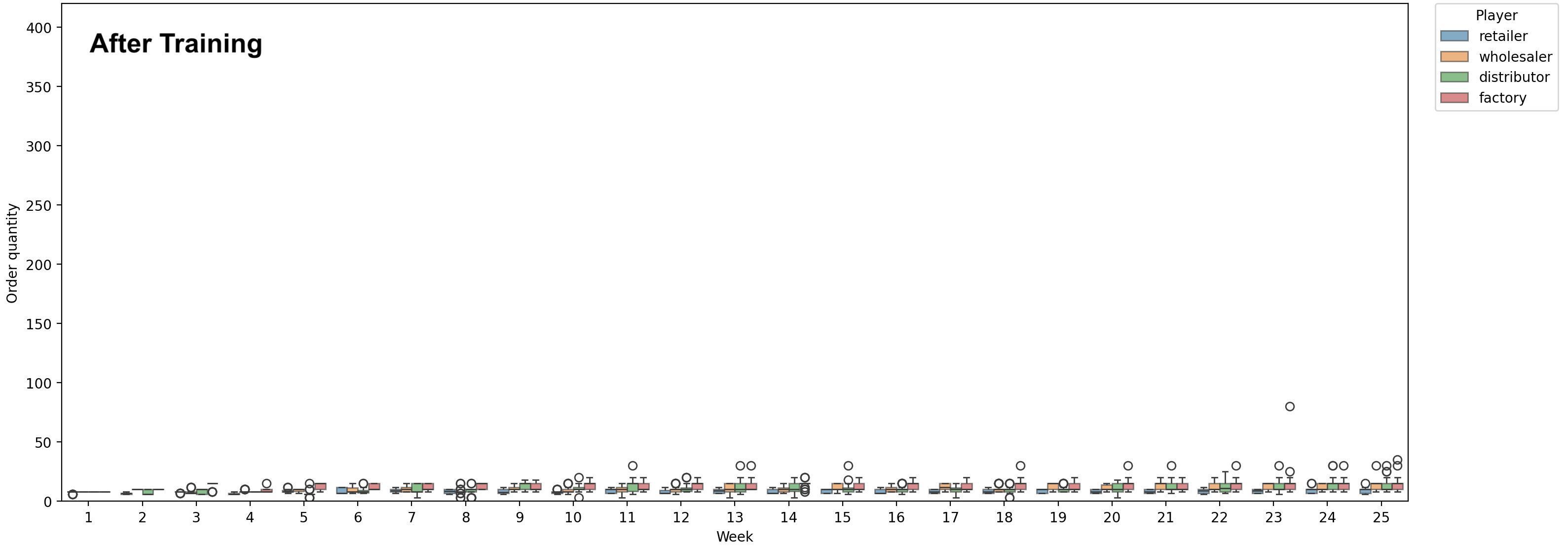}
    \caption{\textbf{Impact of post-training on order reliability.} Post-training significantly compresses decision variance across all facilities and mitigates outlier events. Note: The y-axis scale is held constant with Figures 2 and 3 to facilitate direct comparison.}
    \label{fig:second_order_bullwhip_after_training}
\end{figure}

A central design choice in our GRPO framework is the reward function used for training and model updates. We consider four reward formulations that vary along two dimensions. The first is the level of aggregation: rewards can be based either on total system cost or on the cost incurred by the individual agent. The second is the attribution window: rewards can assign the same episode-level outcome to all actions in an episode, or use rollout-based attribution, where each action is evaluated using the cumulative cost from that week through the end of the game. The latter better captures delayed feedback in supply-chain systems. Among the four formulations, agent-level rollout rewards provide the most granular training signal and yield the strongest empirical performance. We therefore report the main results below using the agent-level rollout reward formulation.

Post-training substantially improves the reliability of LLM agents in inventory management. Across 30 identical runs of the MIT Beer Game under the original demand pattern, the default setting in Figure \ref{fig:second_order_bullwhip_default} exhibits wide interquartile ranges, long whiskers, and numerous extreme tail events in the distribution of orders. By contrast, after post-training the model through reinforcement learning on the training curriculum using the GRPO algorithm described above, Figure~\ref{fig:second_order_bullwhip_after_training} shows that the colored boxes representing the upper and lower percentiles largely disappear. This indicates that variation across repeated runs becomes minimal. The trained model therefore produces much more stable ordering decisions, both over time and across all four echelons of the supply chain.

Equally important, post-training sharply reduces tail risk. In the post-training evaluation, the maximum order observed across all facilities and weeks remains below 100, despite the absence of either an explicit budget constraint or a centralized orchestration layer. This result is especially notable because the earlier analysis showed that out-of-the-box agents often required external guardrails and carefully curated information to prevent panic-induced over-ordering. By contrast, post-training appears to enable the model to internalize a more disciplined and reliable decision policy, thereby reducing the need for external controls.

Figure \ref{fig:supply_chain_costs_trained} extends the analysis from decision behavior to system-level performance by examining total supply chain costs across repeated runs of the 25-week Beer Game. Consistent with the stabilization in ordering decisions, post-training yields a substantial improvement in both efficiency and robustness. For Qwen-3 4B, average total supply chain cost declines from 1,585 without training to 952 after post-training, while the coefficient of variation falls from 26\% to 13\%. Tail risk also contracts sharply, with the maximum realized cost across 30 identical runs decreasing from 2,847 to 1,353. By comparison, other out-of-the-box models exhibit both higher average costs and greater variability, including GPT-5 mini at 3,927 with a 45\% coefficient of variation and a maximum of 8,644, and Llama 4 Maverick 17B at 4,026 with a 52\% coefficient of variation and a maximum of 8,912. These results indicate that post-training not only lowers mean cost, but also substantially narrows the distribution of outcomes across identical runs and reduces exposure to high-cost realizations.

\begin{figure}[bt]
    \centering
    \includegraphics[width=0.75\textwidth]{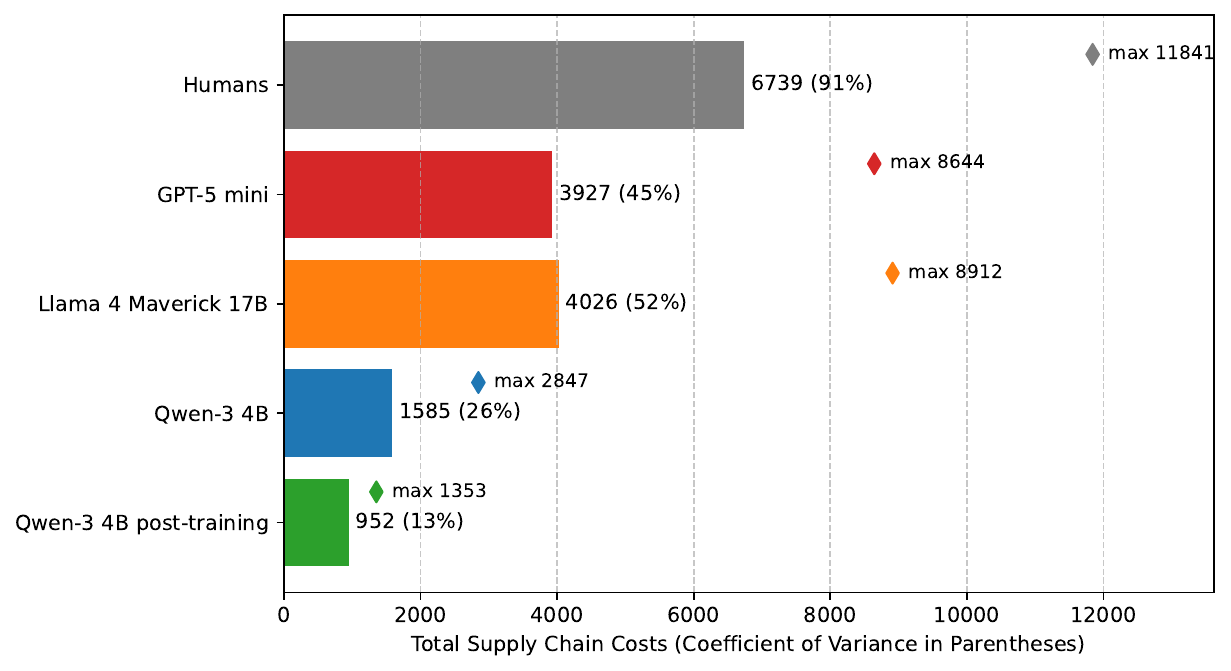}
    \caption{\textbf{Global post-training improves agent reliability across multiple dimensions.} It reduces total supply chain costs, lowers variability across repeated runs, and mitigates worst-case outcomes, thereby improving robustness to tail risks.}
    \label{fig:supply_chain_costs_trained}
\end{figure}

These findings shed light on the source of unreliability in autonomous supply chains. The instability we document does not appear to be an inherent consequence of the stochasticity of language models. Rather, it arises from deploying general-purpose models that have not received specialized training for inventory management. Once the agent is exposed to a curriculum of supply chain tasks with synthetic demand and optimized using realized cost feedback, much of the apparent randomness in its supply chain behavior disappears. The trained agent is both more reliable and more efficient: it makes more consistent decisions across runs, exhibits fewer extreme overreactions, and achieves lower average system-wide costs. This suggests that unreliability is, to a considerable extent, a consequence of insufficient domain specialization rather than irreducible stochasticity, and that specialized post-training can improve both the stability and the economic performance of autonomous supply chain agents.

\subsection{Local Post-Training Fails to Recover System-Level Coordination}
The preceding results show that GRPO post-training can substantially improve the reliability of autonomous supply-chain agents when learning is coordinated across the supply chain. However, system-level cost information may not always be available in practice, particularly when different facilities or firms operate independently. We therefore examine whether GRPO post-training can remain effective under fully decentralized learning, where each echelon trains on its own local cost signals rather than on total supply-chain costs.

This setting differs from the global framework in Section 5.1 along two dimensions. First, each agent is trained using only its own holding and backlog costs, rather than a system-level objective that aggregates costs across the supply chain. Second, instead of updating a shared LLM backbone using experience from all four echelons, we train a separate model for the retailer, wholesaler, distributor, and factory. In the global framework, gradients from all agents are pooled into a common policy update, allowing the model to internalize cross-echelon interactions while still operating with local information at deployment. In the decentralized setting, neither rewards nor model updates are shared across agents.

\begin{figure}[ht]
    \centering
    \includegraphics[width=0.85\textwidth]{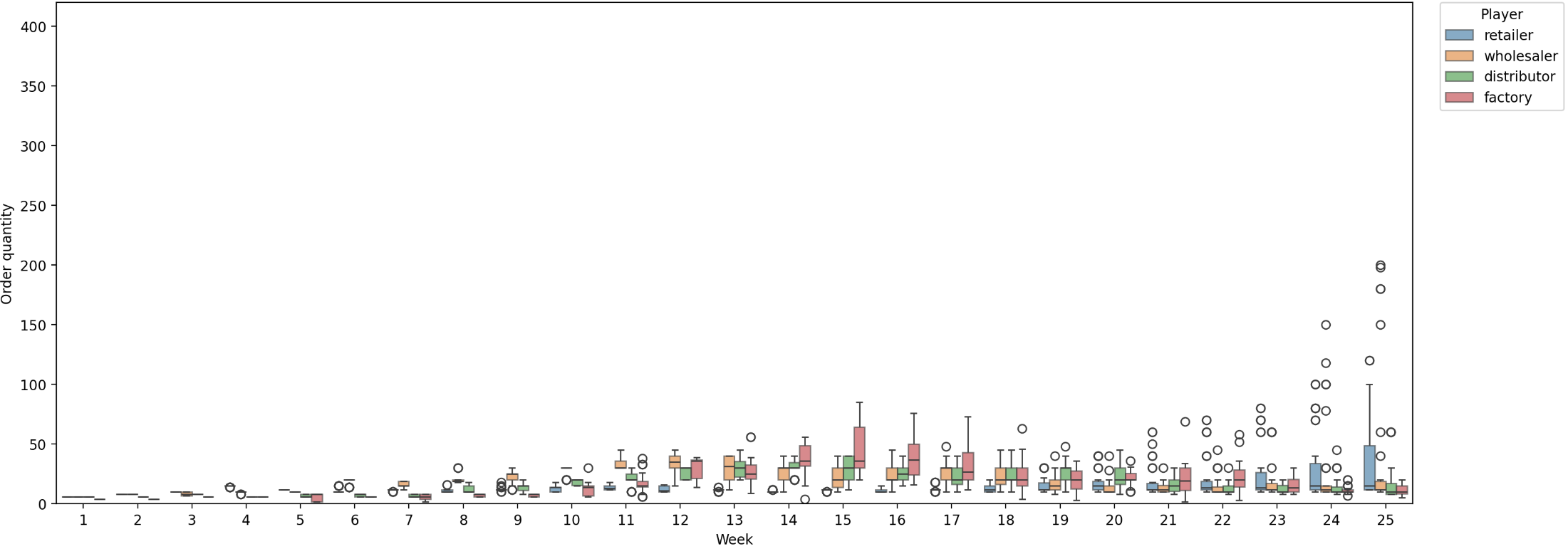}
    \caption{\textbf{Impact of local post-training on order reliability.} Local post-training enables a mild decrease in decision variance across facilities but fail to mitigate outlier events. Note: The y-axis scale is held constant with Figures 2-5 to facilitate direct comparison.}
    \label{fig:second_order_bullwhip_after_local_training}
\end{figure}

This design provides a direct test of whether local optimization is sufficient for system-level coordination. Figure \ref{fig:second_order_bullwhip_after_local_training} answers in the negative: substantial decision instability remains under decentralized post-training, with wide dispersion and extreme orders appearing in later periods. The result suggests that improving each agent with respect to its own local objective does not converge to a globally stable supply-chain policy.

The comparison therefore highlights an important distinction between decentralized execution and decentralized learning. Global GRPO preserves decentralized execution but uses shared training signals to internalize system-wide interactions. Fully decentralizing the learning process removes this coordination mechanism and leads to substantially weaker reliability. This motivates the next question: whether simple operational guardrails can provide a more robust reliability intervention when coordinated post-training is unavailable.

\subsection{Guardrails and Global Post-Training Offer Distinct Reliability Interventions}

We revisit guardrails as a policy-level intervention that is both effective and easy to implement. Rather than retraining the model, guardrails impose explicit operational constraints on the agent’s feasible actions to prevent extreme actions. This makes guardrails particularly attractive when post-training is costly or infeasible, especially for closed-weight models for which the underlying parameters are not accessible (e.g., GPT models). In our experiment, we impose a simple guardrail that limits agents’ total expenditures to twice the total customer demand. As discussed in Section \ref{subsec:guardrail}, a simple budget constraint substantially reduced average supply-chain costs by 3\% to 41\% across the models we evaluated. 
Importantly, guardrail design is itself a non-trivial operational problem. When the underlying demand is well understood, appropriate limits can be calibrated relatively easily; in practice, however, demand is uncertain and evolving. Overly permissive constraints provide limited protection against tail events, while overly restrictive constraints can impair performance. 

We now examine how guardrails interact with progressively stronger policy interventions --- from the default model to decentralized and global post-training --- and whether layering explicit operational constraints on top of learned policies yields further reductions in tail risk and improvements in reliability.


Figure \ref{fig:interventions_ladder} compares several levels of intervention for Qwen-3 4B: the default model, decentralized post-training using local objectives, and global post-training, each with and without guardrails. The comparison reveals a clear hierarchy. Decentralized post-training performs poorly despite the additional training, indicating that optimizing agents independently on local costs does not recover effective system-level coordination. By contrast, global post-training substantially reduces both average cost and tail risk, consistent with the value of learning from coordinated supply-chain objectives. Importantly, adding guardrails further improves reliability across all three model configurations. For the default model, the guardrail reduces average cost by 22\%, lowers the coefficient of variation from 26\% to 12\%, and reduces the maximum observed cost by 45\%. 

\begin{figure}[bt]
    \centering
    \includegraphics[width=0.8\textwidth]{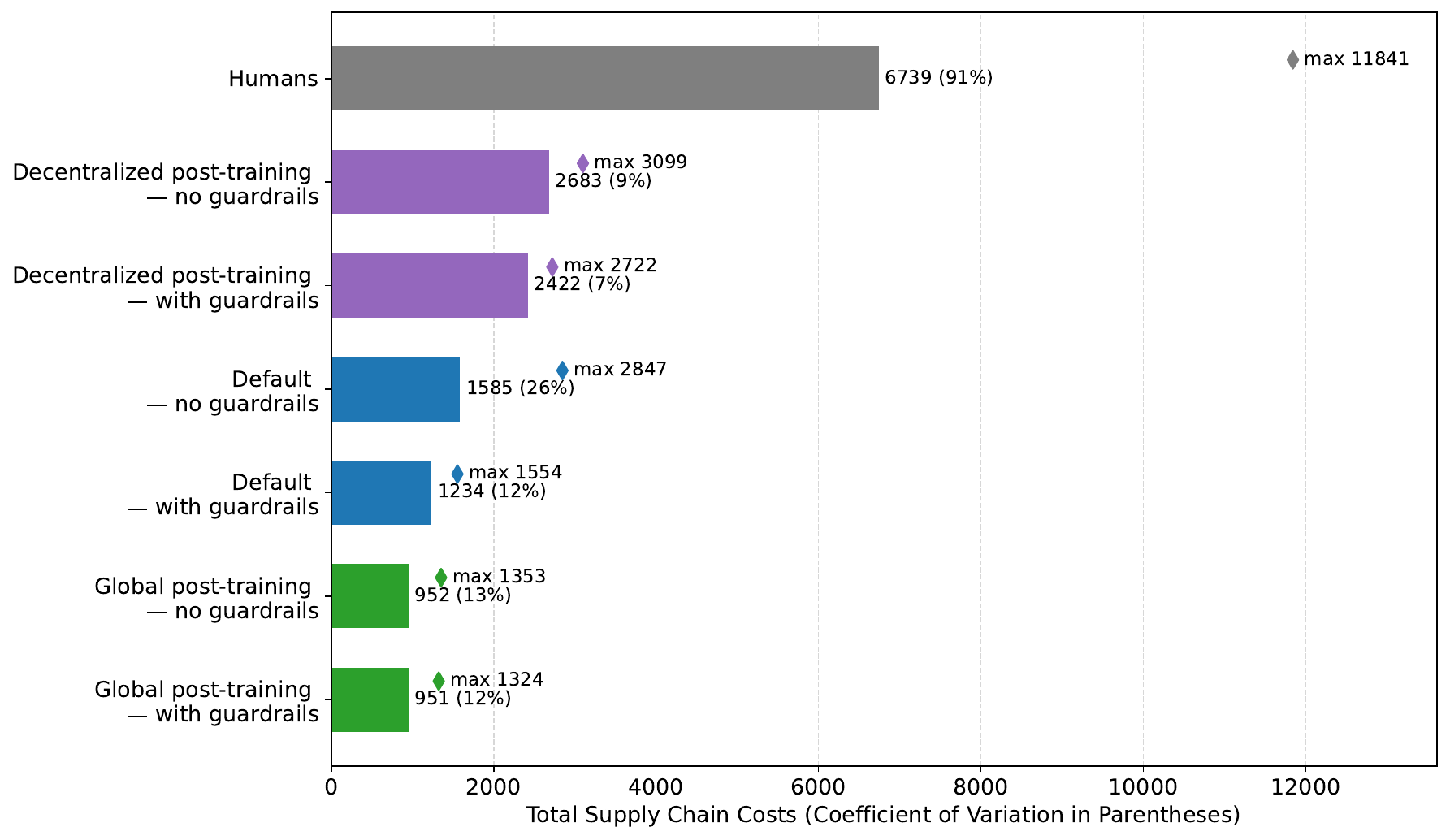}
    \caption{\textbf{A clear performance hierarchy emerges across policy interventions.} Global post-training with guardrails performs best, followed by the default model with guardrails, while decentralized post-training under local objectives underperforms the untrained default model. Mean total supply-chain costs of the 25-week Beer Game are shown for human teams and Qwen-3 4B under default, decentralized post-training, and global post-training, with and without guardrails. Parentheses report coefficients of variation, and diamonds mark maximum observed costs across 30 runs. Guardrails reduce tail risk across all model configurations, while decentralized post-training with local objectives performs substantially worse than both global post-training and the default model.}
    \label{fig:interventions_ladder}
\end{figure}

These results highlight the contrasting pathways of global post-training and guardrails. Guardrails provide a simple, training-free mechanism for improving reliability and limiting extreme outcomes, making them a practical intervention even when model training is unavailable. GRPO-based post-training delivers the most substantial gains, but its effectiveness depends critically on whether centralized cost information is available.


\subsection{Summary of Results}
Our results point to two practical frameworks for improving the reliability of autonomous supply-chain agents: training-free guardrails and global post-training. Guardrails impose explicit operational limits on agent actions, such as budget or ordering constraints, and can be applied directly to an existing model without retraining. Global post-training instead modifies the underlying decision policy using system-wide supply-chain feedback, allowing the model to internalize more stable and coordinated behavior across echelons. These approaches address the same reliability problem through different mechanisms: guardrails restrict extreme actions, whereas post-training enables agents to learn a coordinated decision policy.

Guardrails are the easier of the two approaches to deploy, but their design should not be viewed as trivial. Their effectiveness depends on selecting constraints that are sufficiently tight to prevent extreme actions without unduly restricting the agents' state space. An overly permissive guardrail provides little protection, whereas an overly restrictive one may prevent the agent from responding effectively to genuine demand or supply shocks. Thus, although guardrails require no model training and minimal runtime computation, identifying appropriate operational limits is non-trivial.

The empirical results reveal a clear hierarchy across interventions. Decentralized post-training, with or without guardrails, does not provide an effective substitute for coordinated learning. Training each echelon separately using only its local cost fails to converge to an effective policy and leads to substantially weaker system-level performance. By contrast, applying a guardrail directly to the untrained default model provides a simple and effective intervention: relative to the default model without guardrails, the budget constraint reduces total supply-chain cost by approximately 22\% and cuts the coefficient of variation from 26\% to 12\%, effectively halving run-to-run variability. This result is particularly important because it demonstrates that meaningful reliability improvements can be achieved without any task-specific model training.

\begin{table}[htbp]
\centering
\caption{Comparison of Guardrails and Global Post-Training}
\label{tab:guardrail_vs_posttraining}
\small 

\renewcommand{\arraystretch}{1.35}
\setlength{\tabcolsep}{6pt}

\begin{tabularx}{\textwidth}{
    >{\raggedright\arraybackslash}p{0.18\textwidth}
    >{\raggedright\arraybackslash}X
    >{\raggedright\arraybackslash}X
}
\toprule
& \textbf{Guardrail (Budget Constraint)}
& \textbf{Global Post-Training} \\
\midrule

\textbf{How it works}
& Sets explicit limits on agent actions, such as spending or order caps.
& Trains agents to learn a coordinated supply-chain policy. \\

\textbf{Upfront requirements}
& No model training required.
& Compute for LLM post-training; System-wide supply chain costs info
\\

\textbf{At deployment}
& Each action is checked against predefined limits; minimal additional compute.
& No additional control step required. \\
\textbf{Key challenge}
& Choosing appropriate limits without overly constraining decisions.
& Accessing system-wide data and supporting model post-training. \\

\textbf{Best suited for}
& Fast, low-complexity deployment; when model retraining is unavailable.
& Organizations with sufficient data, compute, and ability to retrain models. \\

\textbf{Observed impact}
& Improves performance and materially reduces tail risk.
& Delivers the strongest gains in performance and reliability. \\

\bottomrule
\end{tabularx}
\end{table}

Global post-training delivers the strongest overall performance and reliability, reflecting the value of learning from system-wide cost information and coordinating behavior across echelons during training. The key challenges lie in training compute as well as the availability of system-wide cost information. In settings where these requirements can be met, global post-training is the preferred intervention. Where they cannot, guardrails provide an effective and substantially easier-to-implement alternative.

Importantly, our results suggest that firms do not need to deploy both mechanisms simultaneously. Adding guardrails on top of a globally post-trained model yields only incremental gains. The more consequential choice is therefore between global post-training, when centralized information and training capabilities are available, and guardrails when they are not. Guardrails offer the lowest barrier to adoption, while global post-training provides the strongest reliability gains when the necessary infrastructure and information are available.

\section{Conclusion: New Paradigm for Supply Chain Management}\label{sec:Conclusion}
GenAI has brought supply chain management to an inflection point: a fully autonomous supply chain is moving rapidly from theory to practice. Early AI models lacked the reasoning required for complex, strategic supply-chain decisions; new-generation models have closed that gap.
We show that GenAI agents can exceed human decision-making in planning and replenishment --- state-of-the-art models outperform human teams operating the same system. With the right combination of model selection, prompts, guardrails, and orchestrated information sharing, autonomous agents reduce total costs by up to 80\%.

However, out-of-the-box agents remain unreliable in critical ways. They can exhibit high order volatility both across facilities and over time, even under the same demand path. Repeated sampling does not meaningfully reduce it, indicating that reliability requires changing the underlying decision policy rather than simply averaging model outputs.

We address this challenge with two complementary interventions. GRPO-based post-training substantially improves both performance and reliability by adapting a shared base model toward more stable and coordinated inventory policies. When system-level rewards are unavailable, inference-time guardrails provide a training-free alternative that constrains extreme actions. Both interventions reduce tail events and mitigate agent bullwhip, with GRPO delivering the largest overall gains.

These findings suggest that advancing autonomous supply chains requires more than increasing LLM reasoning capability: it requires developing agents that make consistently reliable decisions as uncertainty and interdependencies unfold over time. For firms evaluating GenAI for planning and replenishment, average performance is therefore not enough: operational readiness requires both strong expected performance and stable tail behavior. 
As AI adoption advances, supply-chain leaders will shift from hands-on management to strategically orchestrating GenAI agents by selecting appropriate models, curating information, and designing effective training and safeguards for reliability. Mastering these is the new playbook for autonomous supply chains.

\section*{Acknowledgements}
The paper expands and provides more technical details on the concepts and framework described in a recent article: Long, C., Simchi-Levi, D., Calmon, A. P., \& Calmon, F. P. When supply chains become autonomous. Harvard Business Review \citep{long2025supply}.

This material is based upon work supported by the National Science Foundation
under Grant No FAI 2040880, CIF 2231707, and CIF 2312667. F. P. Calmon and C. Long would also like to acknowledge support from Coefficient Giving and JPMorgan Chase. F.P. Calmon is also affiliated with Google Research as a Visiting Faculty Researcher. This research was also supported by the ETH AI Center through an ETH AI Center postdoctoral fellowship to C. Long.

The authors would like to acknowledge support from Harvard Information Theory Lab, MIT Data Science Lab, Microsoft Accelerating AI Academic Research (AAAR) program, the Kempner Institute at Harvard University, and the Ray C. Anderson Center for Sustainable Business at Georgia Tech.



\bibliographystyle{plainnat}
\bibliography{main}

\onecolumn
\appendix
\IfFileExists{appendix.tex}{
\section{Additional Related Works} 

\paragraph{Multi-Agent Collaboration} Multi-agent collaboration has long been studied in supply-chain management. Early agent-based work treated supply chains as networks of autonomous but interdependent entities and focused on modeling, coordination, and decision support. Swaminathan et al. \citep{swaminathan1998modeling} develop a reusable multi-agent framework in which supply-chain models are built from agent types, control elements, and interaction protocols. Fox et al. \citep{fox2000agent} propose an agent-oriented architecture for tactical and operational supply-chain management, emphasizing autonomous software agents that coordinate through communication protocols. Nissen et al. \citep{nissen2001agent} study intelligent agents for supply-chain integration, focusing on agents that conduct business on behalf of buyers, vendors, and users. Julka et al. \citep{julka2002agent} apply an agent-based decision-support framework to refinery supply-chain management, where coordination across departments and dynamic data sources is central.

More recent work extends this tradition using learning-based and foundation-model approaches. Multi-agent reinforcement-learning studies train decentralized or partially decentralized policies for inventory control and transshipment, often using centralized training to improve coordination while preserving decentralized execution~\citep{kotecha2025leveraging,kim2024multi}. A parallel stream explores LLM-based agents for supply-chain tasks: Quan et al. \citep{quan2024invagent} introduce InvAgent, a zero-shot LLM-based multi-agent system for inventory management; Jannelli et al.  \citep{jannelli2026agentic} study autonomous LLM agents for consensus-seeking in supply-chain coordination; and Xu et al. \citep{xu2024implementing} examine autonomous supply chains through a multi-agent systems approach. Our work builds on these streams but shifts the focus from modeling architectures, average performance, or task automation to reliability. We study autonomous LLM agents whose decisions interact dynamically through inventories, backlogs, orders, and shipment pipelines, and show that these interactions can amplify decision instability across repeated runs even when demand paths and operational states are held fixed.

\paragraph{Sources of Bullwhip Effect}
For completeness, we briefly review the established literature on the sources of the bullwhip effect, which has been the focus of extensive prior research. Although these mechanisms are not the focus of this paper, they provide useful background on the factors that can give rise to the bullwhip effect in supply chains managed by human decision-makers.
The literature identifies both operational and behavioral mechanisms underlying the bullwhip effect. From an operational perspective, demand-signal processing, order batching, price fluctuations, and rationing incentives can amplify order variability as demand information propagates upstream\cite{lee1997bullwhip}. 
Forecasting and replenishment policies can further amplify variability, with the magnitude of the effect depending on demand dynamics, lead times, and the structure of the ordering policy\citep{chen2009information, warburton2004analytical}. 
Empirical evidence further shows that the prevalence and magnitude of the bullwhip effect vary across operating environments\citep{cachon2007search}. Beyond these structural mechanisms, experimental studies identify systematic behavioral sources of amplification. \cite{croson2006behavioral} show that bullwhip persists even when conventional operational causes are removed and attribute part of the effect to decision makers' underweighting of the supply line. \cite{croson2014order} show that coordination risk, the uncertainty about other supply-chain members’ decisions, can increase order oscillation and amplification. \cite{moritz2022unraveling} further distinguish between rational and behavioral ordering and examine how each contribute to the bullwhip effect.

\section{Additional Information for Section \ref{sec:inference-time_methods_efficiency}}

\subsection{Detailed Experimental Results}
The following tables present the underlying numerical data that support the findings
discussed in the main text. Table~\ref{tab:human_teams_costs} reports the total supply chain costs recorded across eleven runs of the Beer Game played by
human student teams (4-8 students per team, 100+ students in total) from two academic years in Georgia Tech (April 2025 and April 2024),
together with the average cost of \$3,207 for 20-week and \$6,739 for 25-week that serves as the human performance
benchmark throughout the study. 

\paragraph{Choice of Time Horizon for Different Analyses} Because half of the cohorts play a 25-week game and the other half play a 20-week game, we use all cohorts (excluding one anomalous cohort) when comparing average performance across human and AI teams. However, because reliability risks become more pronounced over longer time horizons, we use the 25-week games separately to evaluate the effectiveness of the two reliability intervention frameworks. This longer horizon is also more representative of real-world supply chain decision-making, which typically extends well beyond the time horizon considered in our stylized experimental setting. Although the available human data for the 25-week games are limited, we include these observations as a reference point.

Table~\ref{tab:supply_chain_full} summarizes the
aggregate performance of all gen AI configurations tested under the classical Beer Game setting played by students (20-week, 2-2-2 lead times), listing total costs and normalized costs relative to the human benchmark
for each combination of model type and inference-time technique; values below 100 indicate that the agent outperformed the average human team. Table~\ref{tab:supply_chain_analysis} identifies the specific model comparisons underlying the numerical results reported in the main text, together with the corresponding percentage changes in total supply chain costs and differences in the coefficient of variance.


\begin{table}[ht]
    \centering
    \caption{Human Teams: Total Costs per Run. Brackets indicate anomaly that are dropped from the mean calculation.}
    \label{tab:human_teams_costs}
        \begin{tabular}{lrr}
            \toprule
            \textbf{Human Teams} & \textbf{Total Costs 20 Weeks} & \textbf{Total Costs 25 Weeks} \\
            \midrule
            4/6/2025, 50 students &   867.5 & - \\
                                  &   854.0 & -\\
                                  &  1091.0 & -\\
                                  &  8784.5 & -\\
                                  & (208920)  & - \\
            4/3/2024, 60 students &  1049.5 & 2201 \\
                                  &   695.5 & 1582.5\\
                                  &  4732.5 & 14755\\
                                  &  7182.5 & 11840.5 \\
                                  &  6258.5 & (79878.5)\\
                                  &  2024.5 & 3317.5\\
                                  &  1735.0 & 3738.5 \\
            \midrule
            Average Costs          & 3206.82  & 6739.3  \\
            \bottomrule
        \end{tabular}
\end{table}

\begin{table}[ht]
    \centering
    \caption{Supply Chain Performance: 20-Week Beer Game}
    \label{tab:supply_chain_full}
    \begin{adjustbox}{width=\textwidth}
        \begin{tabular}{llrr}
            \toprule
            \textbf{Scenario} & \textbf{Model/Human} & \textbf{Total Costs} & \textbf{Normalized Costs} \\
            \midrule
            Human                              & Human Teams                            & 3207  & 100.00  \\
            Reasoning Model                    & GPT-5 mini default                     & 2142  & 66.79   \\
            Reasoning + Orchestrator           & GPT-5 mini demand                      & 1765  & 55.04   \\
            Reasoning + Policy                 & GPT-5 mini budget                      & 1608  & 50.14   \\
            Reasoning + Prompt                 & GPT-5 mini prompt                      & 2620  & 81.70   \\
            Reasoning + Orchestrator + Policy  & GPT-5 mini demand+budget               & 1596  & 49.77   \\
            Reasoning Model                    & Llama 4 Maverick 17B default           & 2080  & 64.86   \\
            Reasoning + Orchestrator           & Llama 4 Maverick 17B demand            & 1559  & 48.61   \\
            Reasoning + Policy                 & Llama 4 Maverick 17B budget            & 1235  & 38.51   \\
            Reasoning + Prompt                 & Llama 4 Maverick 17B prompt            & 1758  & 54.82   \\
            Reasoning + Orchestrator + Policy  & Llama 4 Maverick 17B demand+budget     & 1063  & 33.15   \\
            Reasoning Model                    & Qwen2.5 3B Instruct default           & 728  & 22.70   \\
            Reasoning + Orchestrator           & Qwen2.5 3B Instruct demand            & 678  & 21.14   \\
            Reasoning + Policy                 & Qwen2.5 3B Instruct budget            & 708  & 22.07   \\
            Reasoning + Prompt                 & Qwen2.5 3B Instruct prompt            & 1236  & 38.54   \\
            Reasoning + Orchestrator + Policy  & Qwen2.5 3B Instruct demand+budget     & 652  & 20.33   \\
            Non-Reasoning Model                & GPT-4o mini default                    & 7093  & 221.17  \\
            Non-Reasoning + Orchestrator       & GPT-4o mini demand+volatility          & 2171  & 67.70   \\
            Non-Reasoning + Policy             & GPT-4o mini budget                     & 4351  & 135.67  \\
            Non-Reasoning + Prompt             & GPT-4o mini prompt                     & 3956  & 123.36  \\
            Non-Reasoning + Orchestrator       & GPT-4o mini demand                     & 4383  & 136.67  \\
            Non-Reasoning Model                & GPT-4.1 mini default                   & 7093  & 221.17  \\
            Non-Reasoning + Prompt             & GPT-4.1 mini prompt                    & 4734  & 147.61  \\
            Non-Reasoning Model                & Llama 3.3 70B default                  & 11743 & 366.17  \\
            \bottomrule
        \end{tabular}
    \end{adjustbox}
\end{table}

\begin{table}[ht]
    \centering
    \caption{Supply Chain Analysis Summary}
    \label{tab:supply_chain_analysis}
    \begin{adjustbox}{width=\textwidth}
        \begin{tabular}{llrr}
            \toprule
            \textbf{Analysis} & \textbf{Comparison} & \textbf{Numbers} & \textbf{Note} \\
            \midrule
            Human vs AI + Orchestration              & Llama 4 Maverick 17B demand+budget vs Human Teams           & $-67\%$ & Percentage Change in Total Costs \\
            Stability in Runs                        & Min: GPT-4o mini default; Max: Llama 4 Maverick 17B default & 13\%, 46\%  & Coefficient of Variation          \\
            Reasoning vs Non-Reasoning               & GPT-4o mini default vs GPT-5 mini default                   & $-70\%$ & Percentage Change in Total Costs \\
           & Llama 3.3 70B default vs Llama 4 Maverick 17B default       & $-82\%$ & Percentage Change in Total Costs \\
            Effect of Guardrail                      & GPT-5 mini default vs GPT-5 mini budget                     & $-25\%$ & Percentage Change in Total Costs \\
             & GPT-4o mini default vs GPT-4o mini budget                   & $-39\%$ & Percentage Change in Total Costs \\
          & Llama 4 Maverick 17B default vs Llama 4 Maverick 17B budget & $-41\%$ & Percentage Change in Total Costs \\
            Improvement in Stability                 & Llama 4 Maverick 17B default vs Llama 4 Maverick 17B budget & 46\%, 37\%  & Coefficient of Variation          \\
            Orchestration: Demand Sharing            & GPT-5 mini default vs GPT-5 mini demand                     & $-18\%$ & Percentage Change in Total Costs \\
        & Llama 4 Maverick 17B default vs Llama 4 Maverick 17B demand & $-25\%$ & Percentage Change in Total Costs \\
          & GPT-4o mini default vs GPT-4o mini demand                   & $-38\%$ & Percentage Change in Total Costs \\
            Orchestration: Demand Sharing + Analysis & GPT-4o mini default vs GPT-4o mini demand+volatility        & $-69\%$ & Percentage Change in Total Costs \\
            Effect of Prompts                        & GPT-4o mini default vs GPT-4o mini prompt                   & $-44\%$ & Percentage Change in Total Costs \\
        & GPT-4.1 mini default vs GPT-4.1 mini prompt                 & $-33\%$ & Percentage Change in Total Costs \\
            \bottomrule
        \end{tabular}
    \end{adjustbox}
\end{table}

\subsection{LLM Prompts}

\begin{promptbox}[Example Prompt for Retailer]
You are the Retailer in the Beer Distribution Game.
Your objective is to minimize your total supply chain costs by managing your beer inventory efficiently. You receive orders from customers and stock up your inventory from the Wholesaler.
Your only task is to decide, based on your inventory status and incoming order (shown below), how many new cases of beer you want to buy this week.

Here are the costs you face:
- Holding Cost: 0.50 per case per week.
- Backorder Cost: 1.00 per case per week.
- Order Lead Time: 1 week (your order reaches the Wholesaler next week).
- Shipping Lead Time: 2 weeks (your delivery from the Wholesaler arrives 2 weeks after they ship).

**Your Current Situation (Week {week}):**
- Current Inventory: {current_inventory} cases
- Current Backlog: {current_backlog} cases
- Incoming Order from Downstream (Customer Demand): {incoming_order_this_week} cases
- Last Order You Placed: {last_order_placed} cases
- Last Delivery You Received: {last_delivery_received} cases
{pipeline_info}{budget_info}{fixed_cost_info}{order_forecast_info}
{feedback_info}


Your Task:
Decide how many cases of beer to order from your upstream this week based on your current situation.

Start your response with a JSON object **on its own line** in the following exact format:
{"order_quantity": <number_of_cases>}

Important:
- Replace `<number_of_cases>` with your actual numeric decision.
- Do not add any text, notes, or punctuation after the JSON.
- This will be parsed by a program, so the format must be valid and exact.

Example (your response should end like this):
{"order_quantity": 5}


\end{promptbox}

\section{A Theoretical Model for Agent Bullwhip}
\label{sec:theory-bullwhip}

The section proceeds in three steps. First, we introduce a general operational model that includes on-hand inventory, backlog, outstanding inventory, shipment constraints, nonnegative ordering, and decision shocks. Second, we specialize this model to a linear benchmark system that admits exact transfer-function analysis for both demand and decision shocks. Third, within the linear benchmark system, we analyze agent bullwhip by decomposing order variability into demand-driven and decision-driven components. 

\subsection{A General Operational Inventory Model}
\label{subsec:general-operational-model}

Each tier $k=1,\dots,n$ operates under an order-up-to policy with deterministic lead time $\ell_k\ge 0$. For each tier $k$, define the following state variables:
\begin{itemize}
    \item $OH_{k, t}$: on-hand inventory at the beginning of period $t$,
    \item $B_{k, t}$: backlog owed by tier $k$ to tier $k-1$,
    \item $O_{k, t}$: outstanding inventory, i.e., orders already placed by tier $k$ but not yet received,
    \item $IP_{k, t}$: inventory position.
\end{itemize}

The inventory position is defined as
\begin{equation}
IP_{k, t}
=
OH_{k, t}
+
O_{k, t}
-
B_{k, t},
\label{eq:inventory-position-general}
\end{equation}
where inventory position counts physical inventory currently on hand, plus outstanding inventory, minus outstanding backlog.
\paragraph{Demand prediction and order-up-to-level.} Each tier forms an exponentially smoothed forecast of downstream orders observed through period $t-1$, following the timing convention in \citep{chen2000exponential}:
\begin{equation}
\hat q_{k, t}
=
\lambda_k q_{k-1, t-1}
+
(1-\lambda_k)\hat q_{k,t-1},
\qquad
\lambda_k\in(0,1].
\label{eq:forecast-general}
\end{equation}

The corresponding order-up-to target is
\begin{equation}
S_{k, t}
=
\theta_k\hat q_{k, t}
+
\epsilon_{k, t},
\label{eq:order-up-to-general}
\end{equation}
where $\epsilon_{k, t}$ denotes the tier's decision shock: a safety-stock perturbation, an idiosyncratic residual in the agent policy, or a run-specific interpretation of the same state. $\theta_k$ is a multiplier that controls the inventory target level. Under periodic review, the order-up-to level is typically set to cover expected demand over the protection interval, i.e., the replenishment lead time plus the review period. Hence, when the review period is one period and \(\hat q_{k,t}\) is a one-period demand forecast, a common choice is $\theta_{k}=\ell_k+1$ \citep{silver1998inventory, chen2000exponential}.

Thus, the order placed by tier $k$ is given by:
\begin{equation}
q_{k, t}
=
\left[
S_{k, t}
-
IP_{k, t}
\right]^+,
\label{eq:order-general}
\end{equation}
where $[x]^+ := \max\{x,0\}$. The positive-part operator captures the practical constraint that order quantities cannot be negative.

This timing convention means tier $k$ places its period-$t$ order using information observed through period $t-1$. The current downstream order $q_{k-1,t}$ then enters the inventory-position update after the period-$t$ order is placed. Under this convention, the exact transfer function contains a leading one-period delay. Removing that pure delay gives the classical exponential-smoothing bullwhip filter used for variance-gain calculations, because the lag operator has unit modulus in the frequency domain.

\paragraph{System dynamics.} Let $r_{k, t}$ denote inbound receipts to tier $k$, and let $s_{k, t}$ denote shipments from tier $k$ to tier $k-1$. For $k<n$, receipts at tier $k$ are delayed shipments from tier $k+1$: $r_{k, t} = s_{k+1,t-\ell_k}$. For the most upstream tier, we assume access to an outside supplier with unlimited capacity: $r_{n, t} = q_{n,t-\ell_n}$.

The effective demand faced by tier $k$ is the sum of current downstream orders and existing backlog:
\begin{equation}
\Delta_{k, t}
=
q_{k-1, t}
+
B_{k, t}.
\label{eq:effective-demand-general}
\end{equation}
Available inventory at tier $k$ is
\begin{equation}
A_{k, t}
=
OH_{k, t}
+
r_{k, t}.
\label{eq:available-inventory-general}
\end{equation}
Actual shipments are constrained by available inventory:
\begin{equation}
s_{k, t}
=
\min\left\{
A_{k, t},\Delta_{k, t}
\right\}.
\label{eq:shipment-general}
\end{equation}

The operational state dynamics are therefore
\begin{align}
OH_{k,t+1}
&=
OH_{k, t}
+
r_{k, t}
-
s_{k, t},
\label{eq:oh-update-general}
\\
B_{k,t+1}
&=
B_{k, t}
+
q_{k-1, t}
-
s_{k, t},
\label{eq:backlog-update-general}
\\
O_{k,t+1} & = O_{k, t} + q_{k, t} - r_{k, t}.
\label{eq:pipeline-update-general}
\end{align}


The following proposition shows that the detailed operational state equations imply a simple inventory-position balance.

\begin{proposition}[Inventory-position recursion]
\label{prop:inventory-position-recursion}
Under the operational dynamics above, the inventory position of tier $k$ satisfies
\begin{equation}
IP_{k, t+1}
=
IP_{k, t}
+
q_{k, t}
-
q_{k-1, t}.
\label{eq:linear-ip}
\end{equation}
\end{proposition}

The nonlinearity in \eqref{eq:order-general} makes the full model analytically difficult. We therefore next study a linear benchmark model that preserves the central feedback mechanism while allowing closed-form characterization.

\subsection{Linear Benchmark Model}
\label{subsec:linear-benchmark-model}

We now introduce a tractable benchmark system.

\begin{assumption}[Linear benchmark system]
\label{assumption:linear-benchmark}
For the analytical benchmark, assume:
\begin{enumerate}
    \item orders are not truncated at zero;
    \item lead times are deterministic.
\end{enumerate}
\end{assumption}

Under Assumption~\ref{assumption:linear-benchmark}, the ordering rule becomes linear:
\begin{equation}
q_{k, t}
=
\theta_k\hat q_{k, t}
+
\epsilon_{k,t}
-
IP_{k, t},
\label{eq:linear-order}
\end{equation}
where the forecast remains as
\begin{equation}
\hat q_{k, t}
=
\lambda_k q_{k-1, t-1}
+
(1-\lambda_k)\hat q_{k,t-1}.
\label{eq:linear-forecast}
\end{equation}

Let the customer demand \(\{D_t\}_{t\ge 0}\) be i.i.d.\ with mean zero and variance \(\sigma_D^2\), and suppose all initial conditions are deterministic. Since our focus is variance amplification, centering the demand process entails no loss of generality.


The reduced linear benchmark is therefore
\begin{align}
q_{k, t}
&=
\theta_k\hat q_{k, t}
+
\epsilon_{k,t}
-
IP_{k, t},
\label{eq:linear-system-order}
\\
IP_{k,t+1}
&=
IP_{k, t}
+
q_{k, t}
-
q_{k-1, t},
\label{eq:linear-system-ip}
\\
\hat q_{k, t}
&=
\lambda_k q_{k-1, t-1}
+
(1-\lambda_k)\hat q_{k,t-1}.
\label{eq:linear-system-forecast}
\end{align}


\paragraph{The lag operator.}
We work with discrete time series indexed by $t\in\mathbb{Z}$, and use the \emph{lag operator} $\mathcal{L}$ to express linear dynamics compactly. For any time series $x_t$, the lag operator shifts the index back by one period,
\[
    \mathcal{L}\, x_t = x_{t-1}.
\]
Powers of $\mathcal{L}$ iterate this shift, $\mathcal{L}^{j}x_t = x_{t-j}$, and a polynomial (or rational function) of $\mathcal{L}$ acts on $x_t$ in the obvious way; for example, $(1-\mathcal{L})x_t = x_t - x_{t-1}$. Throughout the analysis, equalities involving $\mathcal{L}$ are understood to hold for all $t$.

\paragraph{The order transfer function.} We now derive the transfer representation that maps downstream orders and local decision shocks into the tier-$k$ order.

\begin{proposition}[One-tier transfer function with decision shocks]
\label{prop:one-tier-transfer-function}
Under the linear benchmark system, the order process at tier $k$ satisfies
\begin{equation}
q_{k,t+1}
=
(1+\theta_k\lambda_k)q_{k-1,t}
-
(\theta_k\lambda_k+1-\lambda_k)q_{k-1, t-1}
+
(1-\lambda_k)q_{k,t}
+
\epsilon_{k,t+1}
-
(2-\lambda_k)\epsilon_{k,t}
+
(1-\lambda_k)\epsilon_{k,t-1}.
\label{eq:general-reduced-recursion}
\end{equation}

Equivalently, in lag-operator form,
\begin{equation}
{\bm q}_k
=
H_k(\mathcal L){\bm q}_{k-1}
+
G(\mathcal L){\bm \epsilon}_{k},
\label{eq:general-transfer-form}
\end{equation}
where ${\bm q}_k=\{q_{k,t}\}_{t\ge0}$, ${\bm q}_{k-1}=\{q_{k-1,t}\}_{t\ge0}$, ${\bm \epsilon}_{k}=\{\epsilon_{k,t}\}_{t\ge0}$, $\mathcal L x_t=x_{t-1}$, and
\begin{equation}
H_k(\mathcal L)
=
\frac{
(1+\theta_k\lambda_k)\mathcal L
-
(\theta_k\lambda_k+1-\lambda_k)\mathcal L^2
}{
1-(1-\lambda_k)\mathcal L
}
=
\mathcal L\left[
1+
\frac{\theta_k\lambda_k(1-\mathcal L)}
{1-(1-\lambda_k)\mathcal L}
\right].
\label{eq:general-transfer-function}
\end{equation}
The local decision-shock filter is
\begin{equation}
G(\mathcal L)
=
1-\mathcal L.
\label{eq:decision-shock-transfer-function}
\end{equation}
\end{proposition}


\subsection{Agent Bullwhip and Decomposition}
\label{subsec:variance-amplification-across-tiers}

We now characterize how the order variance changes as orders propagate upstream. In the stationary benchmark, we drop the time subscript from the variance components and write
\[
V^D_k
:=
\operatorname{Var}_{D}\!\left(\mathbb E_{\epsilon}[q_{k,t}\mid D]\right),
\qquad
V^\epsilon_k
:=
\mathbb E_D\!\left[\operatorname{Var}_{\epsilon}(q_{k,t}\mid D)\right].
\]
To make the two components explicit, define the demand-driven conditional mean
\[
\bar q_{k,t}(D)
:=
\mathbb E_{\epsilon}[q_{k,t}\mid D],
\qquad
\bar{\bm q}_k=\{\bar q_{k,t}(D)\}_{t\ge0},
\]
and the centered decision-driven deviation
\[
x_{k,t}(D)
:=
q_{k,t}-\bar q_{k,t}(D),
\qquad
{\bm x}_{k}=\{x_{k,t}(D)\}_{t\ge0}.
\]
Thus \(V^D_k=\operatorname{Var}_D(\bar q_{k,t}(D))\), while \(V^\epsilon_k=\mathbb E_D[\operatorname{Var}_\epsilon(x_{k,t}(D)\mid D)]\). 


For ease of explanation, we impose the following assumption:
\begin{assumption}[Heterogeneous lower bounds and independent inputs]
\label{assumption:heterogeneous-lower-bounds}
For each tier \(k\), the order-up-to multiplier and smoothing parameter satisfy
\[
    \theta_k\ge \theta>0,
    \qquad
    \lambda_k\in[\lambda,1],
\]
where \(\lambda>0\). The demand process is centered and independent
across time:
\[
    \mathbb E[D_t]=0,
    \qquad
    D_t \perp D_s \quad \text{for } t\neq s,
\]
with
\[
    \operatorname{Var}(D_t)\ge \sigma_D^2>0
    \qquad \text{for all }t.
\]
The decision shocks are centered, independent across tiers and time, and
independent of demand:
\[
    \mathbb E[\epsilon_{k,t}]=0,
    \qquad
    \epsilon_{k,t}\perp \epsilon_{j,s}
    \quad \text{unless } (k,t)=(j,s),
\]
and
\[
    \operatorname{Var}(\epsilon_{k,t})\ge \sigma_k^2>0
    \qquad \text{for all } k,t.
\]
\end{assumption}



\subsubsection{Demand bullwhip}

Taking the conditional expectation of \eqref{eq:general-transfer-form} over \(\epsilon\) gives the demand-channel recursion
\begin{equation}
\bar{\bm q}_k
=
H_k(\mathcal L)\bar{\bm q}_{k-1},
\qquad
\bar{\bm q}_0={\bm D}.
\label{eq:demand-channel-recursion}
\end{equation}
Therefore \(\bar{\bm q}_k=\left(\prod_{r=1}^k H_r(\mathcal L)\right){\bm D}\), with the product ordered from downstream to upstream.







\begin{theorem}[Demand bullwhip]
\label{thm:demand-bullwhip-lower-bounds}
Suppose Assumption~\ref{assumption:heterogeneous-lower-bounds} holds. In the
linear benchmark, the demand-driven component satisfies
\[
V^D_k
\ge
\sigma_D^2\Gamma^k,
\]
where
\[
\Gamma
=
1+
2\theta\lambda
+
\frac{2\theta^2\lambda^2}
{2-\lambda}
>
1.
\]
\end{theorem}


When $\theta_k=\theta$, $\lambda_k=\lambda$, and demands are i.i.d.\ with variance $\sigma_D^2$, Theorem \ref{thm:demand-bullwhip-lower-bounds} gives the common-parameter white-noise benchmark. The theorem shows the demand bullwhip: uncertainty in customer demand is amplified as it moves upstream through the replenishment rule, and the demand-driven component grows at least exponentially in the tier index. This bullwhip effect is consistent with those documented in the literature \citep{chen2000quantifying, chen2000exponential}.

\subsubsection{Decision Bullwhip} 

We next isolate the decision-bullwhip component. Conditioning on \(D=d\) removes demand randomness, so the only remaining variation is the agent's run-to-run decision uncertainty. The customer tier has no decision shock, so \(V^\epsilon_0=0\).

Subtracting \eqref{eq:demand-channel-recursion} from \eqref{eq:general-transfer-form} gives the decision-channel recursion
\begin{equation}
{\bm x}_{k}
=
H_k(\mathcal L){\bm x}_{k-1}
+
G(\mathcal L){\bm \epsilon}_{k},
\qquad
{\bm x}_{0}=0.
\label{eq:decision-channel-recursion}
\end{equation}
Thus the decision component is built up from local shocks injected through \(G\) and then propagated upstream by the tier-specific filters \(H_k\).

\begin{theorem}[Decision bullwhip]
\label{thm:decision-bullwhip-lower-bounds}
Suppose Assumption~\ref{assumption:heterogeneous-lower-bounds} holds. In the
linear benchmark model, the decision-driven component satisfies
\begin{equation}
V^\epsilon_k
\ge
2\sum_{j=1}^k
\sigma_j^2
\Gamma^{k-j},
\label{eq:decision-variance-tier-specific-lower-bound}
\end{equation}
where
\[
\Gamma
=
1+
2\theta\lambda
+
\frac{2\theta^2\lambda^2}
{2-\lambda} > 1.
\]




\end{theorem}

Theorem \ref{thm:decision-bullwhip-lower-bounds} characterizes the decision bullwhip: even when the demand path is fixed, local agent-level decision noise accumulates across tiers and is amplified by the same upstream feedback loop that drives the classical demand bullwhip. In the common-parameter benchmark with comparable decision-shock variances, the lower bound becomes a geometric accumulation term; for example, if \(\sigma_j^2=\sigma_\epsilon^2\) for all \(j\), then
\[
V^\epsilon_k
\ge
2\sigma_\epsilon^2
\sum_{m=0}^{k-1}\Gamma^m.
\]

\begin{corollary}[Exponential growth from any downstream decision noise]
\label{cor:decision-bullwhip-single-source}
Suppose Assumption~\ref{assumption:heterogeneous-lower-bounds} holds. In the
linear benchmark model, if there exists a fixed tier \(j_0\) such that
\[
\sigma_{j_0}^2>0,
\]
then, for all \(k\ge j_0\),
\[
V^\epsilon_k
\ge
2\sigma_{j_0}^2\Gamma^{k-j_0}.
\]
\end{corollary}

Corollary \ref{cor:decision-bullwhip-single-source} suggests that any nonzero decision noise source at a fixed downstream tier generates exponential variance growth as it propagates upstream.


\subsubsection{Discussion}



\paragraph{Accumulation of decision noise.} Our results also imply that decision bullwhip can dominate in regimes where agents are tuned for prediction stability. When the smoothing parameter \(\lambda\) is small, the one-tier gain satisfies
\[
    \Gamma
    =
    1+2\theta\lambda+O(\lambda^2),
\]
so demand-driven amplification can be relatively mild over a moderate number of tiers. By contrast, the decision-driven component still accumulates across tiers. If the decision variances are of comparable magnitude, for example
\(\sigma_j^2\asymp \sigma_\epsilon^2\), then
\[
    V^\epsilon_k
    \;\gtrsim\;
    2\sigma_\epsilon^2
    \sum_{m=0}^{k-1}
    \Gamma^m.
\]
In the small-\(\lambda\) regime, this behaves approximately like \(2\sigma_\epsilon^2 k\) over moderate lead times. Thus even when the propagation gain is close to one, residual agent-level randomness can accumulate across the
supply chain and become a major source of order variability.

These results explain why repeated sampling is an incomplete remedy. Majority
voting or best-of-\(n\) sampling may reduce the local decision variance
\(\sigma_j^2\), but unless it drives this variance very close to zero, the remaining
decision noise continues to enter the feedback system and propagate upstream.
The structural source of the agent bullwhip is therefore not merely stochasticity at a single decision point; it is the \emph{interaction} between residual
decision variability, lead times, information delays, and decentralized
replenishment decisions inherent in autonomous supply chains.


\paragraph{Fixed-tier accumulation over time.}
Our results describe how decision variability grows across tiers. A complementary time-domain implication is that, for a fixed facility, decision unreliability also accumulates over time in the finite-horizon linear benchmark.

\begin{proposition}[Intertemporal accumulation of decision unreliability]
\label{result:temporal-decision-bullwhip}
Fix a demand path $D=d$ and consider the finite-horizon linear benchmark
initialized from deterministic initial conditions, with zero shock pre-history.
Assume the decision shocks are centered, independent across tiers and time,
and satisfy
\[
    \operatorname{Var}(\epsilon_{j,t})=\sigma_j^2<\infty.
\]
For a fixed facility $k$, define
\[
    W_{k,t}(d)
    :=
    \operatorname{Var}(q_{k,t}\mid D=d).
\]
Then
\[
    W_{k,t+1}(d)\ge W_{k,t}(d)
    \qquad\text{for all }t.
\]
\end{proposition}

\section{Proof of Section \ref{sec:theory-bullwhip}}

\begin{proof}[Proof of Proposition \ref{prop:inventory-position-recursion}]
By definition,
\begin{equation}
IP_{k,t}
=
OH_{k,t}
+
O_{k,t}
-
B_{k,t}.
\end{equation}

Therefore,
\begin{equation}
IP_{k,t+1}
=
OH_{k,t+1}
+
O_{k,t+1}
-
B_{k,t+1}.
\end{equation}

Using the operational state equations,
\begin{align}
OH_{k,t+1}
&=
OH_{k,t}
+
r_{k,t}
-
s_{k,t},
\\
O_{k,t+1}
&=
O_{k,t}
+
q_{k,t}
-
r_{k,t},
\\
B_{k,t+1}
&=
B_{k,t}
+
q_{k-1,t}
-
s_{k,t}.
\end{align}

Substituting these three equations into the definition of $IP_{k,t+1}$ gives
\begin{align}
IP_{k,t+1}
&=
\left(
OH_{k,t}
+
r_{k,t}
-
s_{k,t}
\right)
+
\left(
O_{k,t}
+
q_{k,t}
-
r_{k,t}
\right)
-
\left(
B_{k,t}
+
q_{k-1,t}
-
s_{k,t}
\right).
\end{align}

Expanding terms,
\begin{align}
IP_{k,t+1}
&=
OH_{k,t}
+
r_{k,t}
-
s_{k,t}
+
O_{k,t}
+
q_{k,t}
-
r_{k,t}
-
B_{k,t}
-
q_{k-1,t}
+
s_{k,t}.
\end{align}

The receipt terms $r_{k,t}$ cancel, and the shipment terms $s_{k,t}$ also cancel. Hence
\begin{align}
IP_{k,t+1}
&=
OH_{k,t}
+
O_{k,t}
-
B_{k,t}
+
q_{k,t}
-
q_{k-1,t}.
\end{align}

Using the definition of inventory position,
\begin{equation}
OH_{k,t}
+
O_{k,t}
-
B_{k,t}
=
IP_{k,t}.
\end{equation}

Therefore,
\begin{equation}
IP_{k,t+1}
=
IP_{k,t}
+
q_{k,t}
-
q_{k-1,t}.
\end{equation}

This proves the claim.
\end{proof}

\begin{proof}[Proof of Proposition \ref{prop:one-tier-transfer-function}]
Let \(a_k:=1-\lambda_k\), \(u_t:=q_{k-1,t}\), \(y_t:=q_{k,t}\), \(f_t:=\hat q_{k,t}\), and \(\epsilon_t:=\epsilon_{k,t}\). The linear order rule gives
\[
IP_{k,t}
=
\theta_k f_t+\epsilon_t-y_t.
\]
Using the inventory-position recursion,
\[
IP_{k,t+1}
=
\theta_k f_t+\epsilon_t-u_t.
\]
Applying the order rule one period forward yields
\begin{equation}
y_{t+1}
=
\theta_k(f_{t+1}-f_t)
+
u_t
+
\epsilon_{t+1}-\epsilon_t.
\label{eq:epsilon-order-difference-proof}
\end{equation}
Since \(f_{t+1}=\lambda_k u_t+a_k f_t\),
\[
f_{t+1}-f_t
=
\lambda_k(u_t-f_t).
\]
Thus
\begin{equation}
y_{t+1}
=
(1+\theta_k\lambda_k)u_t
-
\theta_k\lambda_k f_t
+
\epsilon_{t+1}-\epsilon_t.
\label{eq:epsilon-order-with-forecast-proof}
\end{equation}

It remains to eliminate \(f_t\). Applying \eqref{eq:epsilon-order-difference-proof} one period earlier gives
\[
y_t
=
\theta_k(f_t-f_{t-1})
+
u_{t-1}
+
\epsilon_t-\epsilon_{t-1}.
\]
Using \(f_t-f_{t-1}=\lambda_k(u_{t-1}-f_{t-1})\), we obtain
\[
\theta_k\lambda_k f_{t-1}
=
\theta_k\lambda_k u_{t-1}
+
u_{t-1}
+
\epsilon_t-\epsilon_{t-1}
-
y_t.
\]
The forecast recursion \(f_t=\lambda_k u_{t-1}+a_k f_{t-1}\) then implies
\begin{align}
\theta_k\lambda_k f_t
&=
\theta_k\lambda_k^2u_{t-1}
+
a_k\theta_k\lambda_k f_{t-1} \\
&=
(\theta_k\lambda_k+a_k)u_{t-1}
+
a_k(\epsilon_t-\epsilon_{t-1})
-
a_k y_t.
\label{eq:epsilon-forecast-eliminated-proof}
\end{align}
Substituting \eqref{eq:epsilon-forecast-eliminated-proof} into \eqref{eq:epsilon-order-with-forecast-proof} gives
\[
y_{t+1}
=
(1+\theta_k\lambda_k)u_t
-
(\theta_k\lambda_k+a_k)u_{t-1}
+
a_k y_t
+
\epsilon_{t+1}
-
(1+a_k)\epsilon_t
+
a_k\epsilon_{t-1}.
\]
Returning to the original notation and using \(a_k=1-\lambda_k\) proves \eqref{eq:general-reduced-recursion}.

Writing the recurrence at time \(t\) in lag-operator form gives
\[
\left[1-a_k\mathcal L\right]q_{k,t}
=
\left[(1+\theta_k\lambda_k)\mathcal L-(\theta_k\lambda_k+a_k)\mathcal L^2\right]q_{k-1,t}
+
\left[1-a_k\mathcal L\right](1-\mathcal L)\epsilon_{k,t}.
\]
Dividing by \(1-a_k\mathcal L\) gives
\[
q^{(k)}
=
H_k(\mathcal L)q^{(k-1)}
+
(1-\mathcal L)\epsilon^{(k)},
\]
where \(H_k\) is \eqref{eq:general-transfer-function}. Hence \(G(\mathcal L)=1-\mathcal L\), proving \eqref{eq:general-transfer-form} and \eqref{eq:decision-shock-transfer-function}.
\end{proof}

\begin{lemma}[One-tier gain]
\label{lem:heterogeneous-one-tier-gain}
For tier \(k\), define
\[
g_k(\omega):=
|H_k(e^{-i\omega})|^2 .
\]
Then
\[
g_k(\omega)
=
1+
\frac{
2\theta_k\lambda_k(2-\lambda_k+\theta_k\lambda_k)(1-\cos\omega)
}{
\lambda_k^2+2(1-\lambda_k)(1-\cos\omega)
} \geq 1.
\]

Moreover, the average gain
\[
\Gamma_k
:=
\frac{1}{2\pi}
\int_{-\pi}^{\pi}
g_k(\omega)\,d\omega
\]
is
\[
\Gamma_k
=
1+
2\theta_k\lambda_k
+
\frac{2\theta_k^2\lambda_k^2}{2-\lambda_k}.
\]
Consequently,
\[
\Gamma_k
\ge
\Gamma
:=
1+
2\theta\lambda
+
\frac{2\theta^2\lambda^2}{2-\lambda}
>
1.
\]
\end{lemma}

\begin{proof}[Proof of Lemma \ref{lem:heterogeneous-one-tier-gain}]
Write \(a_k=1-\lambda_k\). Since the leading lag
\(\mathcal L\) has unit modulus on the unit circle, it does not affect the
frequency gain. Hence
\[
g_k(\omega)
=
\left|
1+
\frac{\theta_k\lambda_k(1-e^{-i\omega})}
{1-a_k e^{-i\omega}}
\right|^2 .
\]
Combining terms gives
\[
g_k(\omega)
=
\left|
\frac{
1+\theta_k\lambda_k
-
(a_k+\theta_k\lambda_k)e^{-i\omega}
}{
1-a_k e^{-i\omega}
}
\right|^2 .
\]
Using
\[
|1-a_k e^{-i\omega}|^2
=
\lambda_k^2+2(1-\lambda_k)(1-\cos\omega),
\]
a direct expansion yields
\[
g_k(\omega)
=
1+
\frac{
2\theta_k\lambda_k(2-\lambda_k+\theta_k\lambda_k)(1-\cos\omega)
}{
\lambda_k^2+2(1-\lambda_k)(1-\cos\omega)
}.
\]
The numerator and denominator of the second term are nonnegative, so
\(g_k(\omega)\ge 1\). If \(\omega\neq 0\), then \(1-\cos\omega>0\), and since
\(\theta_k>0\) and \(\lambda_k>0\), the inequality is strict.

Next, expand \(H_k(\mathcal L)\) as an impulse-response sequence. We have
\[
H_k(\mathcal L)
=
\frac{
(1+\theta_k\lambda_k)\mathcal L
-
(\theta_k\lambda_k+1-\lambda_k)\mathcal L^2
}{
1-(1-\lambda_k)\mathcal L
}.
\]
Therefore the impulse coefficients \(h_{k,j}\) of \(H_k\) satisfy
\[
h_{k,0}=0,
\qquad
h_{k,1}=1+\theta_k\lambda_k,
\]
and, for \(j\ge 2\),
\[
h_{k,j}
=
-\theta_k\lambda_k^2(1-\lambda_k)^{j-2}.
\]
By Parseval's identity,
\[
\Gamma_k
=
\frac{1}{2\pi}\int_{-\pi}^{\pi}g_k(\omega)\,d\omega
=
\sum_{j=0}^{\infty}h_{k,j}^2 .
\]
Thus
\[
\Gamma_k
=
(1+\theta_k\lambda_k)^2
+
\theta_k^2\lambda_k^4
\sum_{j=0}^{\infty}(1-\lambda_k)^{2j}.
\]
Since
\[
\sum_{j=0}^{\infty}(1-\lambda_k)^{2j}
=
\frac{1}{1-(1-\lambda_k)^2}
=
\frac{1}{\lambda_k(2-\lambda_k)},
\]
we obtain
\[
\Gamma_k
=
(1+\theta_k\lambda_k)^2
+
\frac{\theta_k^2\lambda_k^3}{2-\lambda_k}.
\]
Equivalently,
\[
\Gamma_k
=
1+
2\theta_k\lambda_k
+
\frac{2\theta_k^2\lambda_k^2}{2-\lambda_k}.
\]
The expression is increasing in both \(\theta_k\) and \(\lambda_k\) on
\(\theta_k>0\), \(\lambda_k\in(0,1]\). Therefore
\[
\Gamma_k
\ge
1+
2\theta\lambda
+
\frac{2\theta^2\lambda^2}{2-\lambda}
=
\Gamma.
\]
Finally, \(\Gamma>1\) because
\(\theta>0\) and \(\lambda>0\).
\end{proof}

\begin{lemma}[Variance lower bound for independent inputs]
\label{lem:independent-input-lower-bound}
Let
\[
A(\mathcal L)=\sum_{j=0}^{\infty}a_j\mathcal L^j
\]
be a square-summable linear filter. Let \(\{Z_t\}\) be centered and independent
across time, with
\[
\operatorname{Var}(Z_t)\ge \sigma^2>0
\qquad \text{for all }t.
\]
Define
\[
Y_t=A(\mathcal L)Z_t=\sum_{j=0}^{\infty}a_j Z_{t-j}.
\]
Then
\[
\operatorname{Var}(Y_t)
\ge
\sigma^2\sum_{j=0}^{\infty}a_j^2.
\]
Equivalently,
\[
\operatorname{Var}(Y_t)
\ge
\frac{\sigma^2}{2\pi}
\int_{-\pi}^{\pi}
|A(e^{-i\omega})|^2\,d\omega .
\]
\end{lemma}

\begin{proof}
Because the inputs are centered and independent across time,
\[
\operatorname{Var}(Y_t)
=
\operatorname{Var}\left(
\sum_{j=0}^{\infty}a_j Z_{t-j}
\right)
=
\sum_{j=0}^{\infty}a_j^2\operatorname{Var}(Z_{t-j}).
\]
Using the lower bound \(\operatorname{Var}(Z_{t-j})\ge \sigma^2\), we obtain
\[
\operatorname{Var}(Y_t)
\ge
\sigma^2
\sum_{j=0}^{\infty}a_j^2.
\]
The frequency-domain expression follows from Parseval's identity:
\[
\sum_{j=0}^{\infty}a_j^2
=
\frac{1}{2\pi}
\int_{-\pi}^{\pi}
|A(e^{-i\omega})|^2\,d\omega .
\]
\end{proof}

\begin{lemma}[Product-gain lower bound]
\label{lem:product-gain-lower-bound}
Let \(\Omega\) be uniformly distributed on \([-\pi,\pi]\). Suppose
\(f_1,\dots,f_m\) are nonnegative functions of \(1-\cos\Omega\) that are
nondecreasing in \(1-\cos\Omega\). Then
\[
\mathbb E\left[\prod_{r=1}^m f_r(\Omega)\right]
\ge
\prod_{r=1}^m \mathbb E[f_r(\Omega)].
\]
In particular,
\[
\frac{1}{2\pi}
\int_{-\pi}^{\pi}
\prod_{r=1}^m g_r(\omega)\,d\omega
\ge
\prod_{r=1}^m \Gamma_r
\ge
\Gamma^m .
\]
\end{lemma}

\begin{proof}
Let \(U=1-\cos\Omega\). Each \(f_r\) is a nonnegative nondecreasing function
of the same scalar random variable \(U\).

For two nondecreasing functions \(f\) and \(h\), let \(U'\) be an independent
copy of \(U\). Then
\[
\operatorname{Cov}(f(U),h(U))
=
\frac12
\mathbb E\left[
(f(U)-f(U'))(h(U)-h(U'))
\right]
\ge 0,
\]
because the two factors always have the same sign. Therefore
\[
\mathbb E[f(U)h(U)]
\ge
\mathbb E[f(U)]\mathbb E[h(U)].
\]
Applying this argument repeatedly gives
\[
\mathbb E\left[\prod_{r=1}^m f_r(U)\right]
\ge
\prod_{r=1}^m \mathbb E[f_r(U)].
\]

By Lemma~\ref{lem:heterogeneous-one-tier-gain}, each \(g_r(\omega)\) is a
nonnegative nondecreasing function of \(1-\cos\omega\). Hence
\[
\frac{1}{2\pi}
\int_{-\pi}^{\pi}
\prod_{r=1}^m g_r(\omega)\,d\omega
\ge
\prod_{r=1}^m
\frac{1}{2\pi}
\int_{-\pi}^{\pi}g_r(\omega)\,d\omega
=
\prod_{r=1}^m \Gamma_r.
\]
Since each \(\Gamma_r\ge \Gamma\), the result follows.
\end{proof}

\begin{proof}[Proof of Theorem \ref{thm:demand-bullwhip-lower-bounds}]
Taking conditional expectation over the decision shocks gives the demand
channel
\[
\bar{\bm q}_k
=
H_k(\mathcal L)\bar{\bm q}_{k-1},
\qquad
\bar{\bm q}_0={\bm D}.
\]
Therefore
\[
\bar{\bm q}_k
=
\left(
\prod_{r=1}^k H_r(\mathcal L)
\right){\bm D}.
\]
Define the composite demand filter
\[
A_k(\mathcal L)
:=
\prod_{r=1}^k H_r(\mathcal L).
\]
Then
\[
|A_k(e^{-i\omega})|^2
=
\prod_{r=1}^k g_r(\omega).
\]

By Lemma~\ref{lem:independent-input-lower-bound}, since \(D_t\) is centered,
independent across time, and satisfies
\(\operatorname{Var}(D_t)\ge \sigma_D^2\),
\[
V^D_k
=
\operatorname{Var}_D(\bar q_{k,t})
\ge
\frac{\sigma_D^2}{2\pi}
\int_{-\pi}^{\pi}
\prod_{r=1}^k g_r(\omega)\,d\omega .
\]
By Lemma~\ref{lem:product-gain-lower-bound},
\[
\frac{1}{2\pi}
\int_{-\pi}^{\pi}
\prod_{r=1}^k g_r(\omega)\,d\omega
\ge
\prod_{r=1}^k \Gamma_r.
\]
Finally, Lemma~\ref{lem:heterogeneous-one-tier-gain} gives
\[
\Gamma_r\ge \Gamma>1
\qquad \text{for all }r.
\]
Hence
\[
V^D_k
\ge
\sigma_D^2
\prod_{r=1}^k \Gamma_r
\ge
\sigma_D^2\Gamma^k.
\]
This proves the claim.
\end{proof}

\begin{proof}[Proof of Theorem \ref{thm:decision-bullwhip-lower-bounds}]
Fix a demand path \(D=d\). Define the centered decision-driven deviation
\[
x_{k,t}(d)
=
q_{k,t}
-
\mathbb E_\epsilon[q_{k,t}\mid D=d].
\]
Subtracting the demand-channel recursion from the full transfer representation
gives
\[
{\bm x}_k
=
H_k(\mathcal L){\bm x}_{k-1}
+
G(\mathcal L){\bm \epsilon}_k,
\qquad
{\bm x}_0=0.
\]
Iterating this recursion gives
\[
{\bm x}_k
=
\sum_{j=1}^k
\left(
\prod_{r=j+1}^k H_r(\mathcal L)
\right)
G(\mathcal L){\bm \epsilon}_j,
\]
where the product is interpreted as the identity operator when \(j=k\).

Define
\[
A_{j,k}(\mathcal L)
:=
\left(
\prod_{r=j+1}^k H_r(\mathcal L)
\right)
G(\mathcal L).
\]
Then
\[
{\bm x}_k
=
\sum_{j=1}^k
A_{j,k}(\mathcal L){\bm \epsilon}_j.
\]

Because the shock processes are independent across tiers, the summands are
mutually independent. Therefore
\[
V^\epsilon_k
=
\operatorname{Var}_\epsilon(x_{k,t}\mid D=d)
=
\sum_{j=1}^k
\operatorname{Var}_\epsilon
\left(
[A_{j,k}(\mathcal L){\bm \epsilon}_j]_t
\right).
\]

For each \(j\), since \(\{\epsilon_{j,t}\}\) is centered and independent over
time with variance lower bounded by \(\sigma_j^2\), the variance of the filtered process is
\[
\operatorname{Var}_\epsilon
\left(
[A_{j,k}(\mathcal L){\bm \epsilon}_j]_t
\right)
\geq
\frac{\sigma_j^2}{2\pi}
\int_{-\pi}^{\pi}
|A_{j,k}(e^{-i\omega})|^2
\,d\omega .
\]
Moreover,
\[
|A_{j,k}(e^{-i\omega})|^2
=
|G(e^{-i\omega})|^2
\prod_{r=j+1}^k
|H_r(e^{-i\omega})|^2
=
b(\omega)
\prod_{r=j+1}^k g_r(\omega).
\]
Thus
\[
V^\epsilon_k
\geq
\frac{1}{2\pi}
\int_{-\pi}^{\pi}
b(\omega)
\sum_{j=1}^k
\sigma_j^2
\prod_{r=j+1}^k g_r(\omega)
\,d\omega.
\]

Next, \(b(\omega)\) and each \(g_r(\omega)\) are nonnegative nondecreasing
functions of \(1-\cos\omega\). Hence the product-gain lower-bound lemma gives
\[
\frac{1}{2\pi}
\int_{-\pi}^{\pi}
b(\omega)
\prod_{r=j+1}^k g_r(\omega)
\,d\omega
\ge
\left(
\frac{1}{2\pi}
\int_{-\pi}^{\pi}b(\omega)\,d\omega
\right)
\prod_{r=j+1}^k
\left(
\frac{1}{2\pi}
\int_{-\pi}^{\pi}g_r(\omega)\,d\omega
\right).
\]
Since
\[
\frac{1}{2\pi}
\int_{-\pi}^{\pi}b(\omega)\,d\omega
=
\frac{1}{2\pi}
\int_{-\pi}^{\pi}
4\sin^2(\omega/2)
\,d\omega
=
2,
\]
and since
\[
\Gamma_r
=
\frac{1}{2\pi}
\int_{-\pi}^{\pi}g_r(\omega)\,d\omega,
\]
we obtain
\[
V^\epsilon_k
\ge
2\sum_{j=1}^k
\sigma_j^2
\prod_{r=j+1}^k \Gamma_r.
\]

Finally, if
\[
\theta_r\ge \theta>0,
\qquad
\lambda_r\in[\lambda,1],
\]
then Lemma \ref{lem:heterogeneous-one-tier-gain} gives
\[
\Gamma_r\ge \Gamma>1
\]
for every \(r\). Hence
\[
\prod_{r=j+1}^k \Gamma_r
\ge
\Gamma^{k-j}.
\]
Therefore
\[
V^\epsilon_k
\ge
2\sum_{j=1}^k
\sigma_j^2
\Gamma^{k-j}.
\]
This proves \eqref{eq:decision-variance-tier-specific-lower-bound}.
\end{proof}

\begin{proof}[Proof of Proposition \ref{result:temporal-decision-bullwhip}]
Conditional on $D=d$, define
\[
    x_{k,t}(d)
    :=
    q_{k,t}-\mathbb E[q_{k,t}\mid D=d].
\]
The deterministic demand path affects only the conditional mean, so
\[
    W_{k,t}(d)
    =
    \operatorname{Var}(x_{k,t}(d)).
\]
The centered linear recursion is
\[
    x_{k,t}(d)
    =
    H_k(\mathcal L)x_{k-1,t}(d)
    +
    G(\mathcal L)\epsilon_{k,t},
    \qquad
    x_{0,t}(d)=0.
\]
Iterating this recursion gives
\[
    x_{k,t}(d)
    =
    \sum_{j=1}^{k}B_{k,j}(\mathcal L)\epsilon_{j,t},
\]
where
\[
    B_{k,j}(\mathcal L)
    :=
    \left(\prod_{h=j+1}^{k}H_h(\mathcal L)\right)G(\mathcal L),
\]
with the empty product equal to one. Write
\[
    B_{k,j}(\mathcal L)
    =
    \sum_{m=0}^{\infty}b_{k,j,m}\mathcal L^m.
\]
Because the system starts from deterministic initial conditions and the
shock pre-history is zero, only shocks from the first $t$ periods contribute
to the period-$t$ deviation. Thus
\[
    x_{k,t}(d)
    =
    \sum_{j=1}^{k}
    \sum_{m=0}^{t-1}
    b_{k,j,m}\epsilon_{j,t-m}.
\]
Independence across tiers and time eliminates all covariance terms, so
\[
    W_{k,t}(d)
    =
    \sum_{j=1}^{k}
    \sigma_j^2
    \sum_{m=0}^{t-1}
    b_{k,j,m}^2.
\]
Similarly,
\[
    W_{k,t+1}(d)
    =
    \sum_{j=1}^{k}
    \sigma_j^2
    \sum_{m=0}^{t}
    b_{k,j,m}^2.
\]
Subtracting yields
\[
    W_{k,t+1}(d)-W_{k,t}(d)
    =
    \sum_{j=1}^{k}\sigma_j^2 b_{k,j,t}^2
    \ge 0.
\]
This proves finite-horizon intertemporal accumulation of decision
unreliability.
\end{proof}






}{}

\end{document}